\documentclass[journal,doublecolumn,10pt]{IEEEtran}
\usepackage{scrextend}

\usepackage[compress]{cite}
\usepackage[cmex10]{amsmath}
\usepackage[caption=false,font=footnotesize,justification=centering]{subfig}
\usepackage[hyphens]{url}

\usepackage{lipsum}
\usepackage{color}
\usepackage{amssymb, amscd}
\usepackage{mathtools}

\usepackage{pgf,tikz}
\usepackage{pgfplots}
\usetikzlibrary{calc,shadows}
\usetikzlibrary{pgfplots.groupplots}
\usepackage{amsthm}

\newtheorem{thm}{Theorem}

\newtheorem{lem}{Lemma}
\newtheorem{definition}{Definition}
\newtheorem{proposition}{Proposition}
\newtheorem{remark}{Remark}

\usepackage{hyperref}
\hypersetup{
    bookmarks=true,         
    unicode=false,          
    pdftoolbar=true,        
    pdfmenubar=true,        
    pdffitwindow=false,     
    pdfstartview={FitH},    
    pdfauthor={Michael Tschannen},     
    pdfsubject={Subject},   
    pdfcreator={Michael Tschannen},   
    pdfproducer={Producer}, 
    pdfnewwindow=true,      
    colorlinks=false,       
    linkcolor=red,          
    citecolor=green,        
    filecolor=magenta,      
    urlcolor=cyan           
}

\usepackage{jabbrv}

\newcounter{AlgCount}

\usepackage{wrapfig}

\usepackage{ifthen}

 \newcommand{\mc}[0]{\mathcal }
 \newcommand{\mb}[0]{\mathbb }

\usepackage{framed}

\newcommand\norm[2][\Tnorm]{\ensuremath{{\left\Vert #2 \right\Vert}_{#1}}}

\newcommand\Tinnerprod{}
\newcommand{\innerprod}[3][\Tinnerprod]{\ifthenelse{\equal{#1}{}}{\ensuremath{\left<#2,#3\right>}}{\ensuremath{\left<#2,#3\right>_{#1}}}}

\newcommand\vect[1]{\mathbf #1}

\definecolor{DarkBlue}{rgb}{0.1,0.1,0.5}
\definecolor{BrickRed}{RGB}{203,65,84}

\newcommand\PR[1]{\ensuremath{ {\mathrm{P}}\!\left[#1\right]}}

\newcommand\Tex{}
\newcommand\EX[2][\Tex]{
\ifthenelse{\equal{#1}{}}{{\mathbb E}\left[#2\right]}{\ensuremath{{\mathbb E}_{#1}\left[ #2\right]}}}

\newcommand\Var[2][\Tex]{
\ifthenelse{\equal{#1}{}}{{\mathrm{Var} }[#2]}{\ensuremath{\mathrm{Var}_{#1}\left[ #2\right]}}}

\newcommand\ignore[1]{}

\newcommand\defeq{\coloneqq}
\newcommand\eqdef{\eqqcolon}

\newcommand{\reals}{\mathbb R} 
 
\newcommand\comp[1]{ \overline{#1}}

\newcommand{\transp}[1]{{#1}^{\top}} 

\renewcommand{\d}{d} 

\renewcommand{\d}{d} 
\newcommand{\X}{\mathcal X}

\newcommand{\vb}{\vect{b}}
\newcommand{\vc}{\vect{c}}

\newcommand{\vv}{\vect{v}}  

\newcommand{\vx}{\vect{x}}  
\newcommand{\vy}{\vect{y}}  
\newcommand{\vz}{\vect{z}}

\newcommand{\mA}{\vect{A}}  
\newcommand{\mB}{\vect{B}} 
\newcommand{\mC}{\vect{C}} 
\newcommand{\mD}{\vect{D}}

\newcommand{\mG}{\vect{G}}
\newcommand{\mH}{\vect{H}}
\newcommand{\mI}{\vect{I}}

\newcommand{\mP}{\vect{P}}
\newcommand{\mQ}{\vect{Q}}
\newcommand{\mR}{\vect{R}}
\newcommand{\mS}{\vect{S}}

\newcommand{\mZ}{\vect{Z}}

\renewcommand{\O}{\mathcal O} 

\renewcommand{\l}{\ell}
\newcommand{\powpar}[2]{{#1}^{(#2)}}
\newcommand{\Xl}{\powpar{X}{\l}}
\newcommand{\Xk}{\powpar{X}{k}}

\newcommand{\Xlt}{\powpar{\tilde X}{\l}}
\newcommand{\Xlc}{\powpar{\check X}{\l}}
\newcommand{\El}{\powpar{U}{\l}}
\newcommand{\xlh}{\powpar{x}{\l}} 
\newcommand{\xkh}{\powpar{x}{k}} 

\newcommand{\sk}{\powpar{s}{k}}
\renewcommand{\sl}{\powpar{s}{\l}}
\newcommand{\skt}{\powpar{\tilde s}{k}}
\newcommand{\slt}{\powpar{\tilde s}{\l}}
\newcommand{\slh}{\powpar{\hat s}{\l}}
\newcommand{\skh}{\powpar{\hat s}{k}}
\newcommand{\ek}{\powpar{e}{k}}
\newcommand{\el}{\powpar{e}{\l}}
\newcommand{\rl}{\powpar{r}{\l}}
\newcommand{\rlt}{\powpar{\tilde r}{\l}}
\newcommand{\rkt}{\powpar{\tilde r}{k}}
\newcommand{\xh}{x} 
\newcommand{\rh}{\hat r}
\newcommand{\rlh}{\powpar{\hat r}{\l}}
\newcommand{\rlc}{\powpar{\check r}{\l}}
\newcommand{\sh}{\hat s}

\newcommand{\mul}{\powpar{\mu}{\l}}

\newcommand{\al}{\powpar{\alpha}{\l}}
\newcommand{\bl}{\powpar{\beta}{\l}}
\newcommand{\cl}{\powpar{\gamma}{\l}}

\newcommand{\perr}{p}
\newcommand{\errwin}{u}

\newcommand{\gh}{\hat g}
\newcommand{\mGh}{{\hat \mG}}
\newcommand{\mGt}{{\tilde \mG}}

\newcommand{\Fstar}{\mc F^\star}
\newcommand{\Fali}{\powpar{\mc F}{\l}_{1,i}}
\newcommand{\Fbli}{\powpar{\mc F}{\l}_{2,i}}
\newcommand{\Falic}{\powpar{\bar{\mc F}}{\l}_{1,i}}
\newcommand{\Fblic}{\powpar{\bar{\mc F}}{\l}_{2,i}}

\newcommand{\cX}{\mc X}

\renewcommand{\d}{\mathrm{d}}
\newcommand{\im}{\mathrm{i}}
\newcommand{\Wl}{\powpar{W}{\l}}

\newcommand{\abs}[1]{\left\vert #1 \right\vert}
\newcommand{\abss}[1]{\vert #1 \vert}

\newcommand{\E}[1]{\mb E \! \left[ #1 \right]}
\renewcommand{\EX}[2]{\mb E_{#1} \! \left[ #2 \right]}

\newcommand{\Ck}{\powpar{\mC}{k}}
\newcommand{\Cl}{\powpar{\mC}{\l}}
\newcommand{\Rl}{\powpar{\mR}{\l}}
\newcommand{\Rlt}{\powpar{\tilde \mR}{\l}}
\newcommand{\Rkt}{\powpar{\tilde \mR}{k}}

\newcommand{\vxk}{\powpar{\vx}{k}}
\newcommand{\vxl}{\powpar{\vx}{\l}}

\newcommand{\vyl}{\powpar{\vy}{\l}}

\newcommand{\sigkl}{{\powpar{\sigma}{k,\l}}}
\newcommand{\sigll}{{\powpar{\sigma}{\l,\l}}}

\newcommand{\erf}{\mathrm{erf}}
\newcommand{\atan}{\mathrm{arctan}}

\newcommand{\tr}{\mathrm{tr}}

\newcommand{\bxi}{\boldsymbol \xi}
\newcommand{\bnu}{\boldsymbol \nu}

\newcommand{\So}{S_\mathrm{off}}
\newcommand{\Sonu}{S_\nu}

\newcommand{\mSk}{\powpar{\mS}{k}}
\newcommand{\mSl}{\powpar{\mS}{\l}}
\newcommand{\mXl}{\powpar{\vx}{\l}}
\newcommand{\mXk}{\powpar{\vx}{k}}

\begin{document}
%

\title{Robust Nonparametric Nearest Neighbor \\ 
Random Process Clustering}

\author{Michael Tschannen,~\IEEEmembership{Student Member,~IEEE}, and Helmut B{\"o}lcskei,~\IEEEmembership{Fellow,~IEEE} \thanks{The authors are with the Department of Information Technology and Electrical Engineering, ETH Zurich, Switzerland (e-mail: michaelt@nari.ee.ethz.ch; boelcskei@nari.ee.ethz.ch).

Part of this paper was presented at the 2015 IEEE International Symposium on Information Theory (ISIT) \cite{tschannen2015nonparametric}.}}


\maketitle
\IEEEpeerreviewmaketitle

\begin{abstract}

We consider the problem of clustering noisy finite-length observations of stationary ergodic random processes according to their generative models without prior knowledge of the model statistics and the number of generative models. Two algorithms, both using the $L^1$-distance between estimated power spectral densities (PSDs) as a measure of dissimilarity, are analyzed. 
The first one, termed nearest neighbor process clustering (NNPC), relies on partitioning the nearest neighbor graph of the observations via spectral clustering. The second algorithm, simply referred to as $k$-means (KM), 
consists of a single $k$-means iteration with farthest point initialization and was considered before in the literature, albeit with a different dissimilarity measure. 
We prove that both algorithms succeed with high probability in the presence of noise and missing entries, and even when the generative process PSDs overlap significantly, all provided that the observation length is sufficiently large. Our results quantify the tradeoff between the overlap of the generative process PSDs, the observation length, the fraction of missing entries, and the noise variance. Finally, we provide extensive numerical results for synthetic and real data and find that NNPC outperforms state-of-the-art algorithms in human motion sequence clustering.
\end{abstract}





%

\section{Introduction}

Consider a set of $N$ noisy length-$M$ observations of stationary ergodic discrete-time random processes stemming from $L < N$ (typically $L \ll N$) different generative processes, referred to as generative models henceforth. We want to cluster these observations according to their generative models without prior knowledge of the model statistics and the number of generative models, $L$.
This problem arises in many domains of science and engineering where (large amounts of) data have to be divided into meaningful categories in an unsupervised 
fashion.  
Concrete examples include 
audio and video sequences \cite{wang2000multimedia}, electrocardiography (ECG) recordings \cite{kalpakis2001distance}, industrial production indices~\cite{corduas2008time}, and financial time series \cite{marti2016clustering,marti2017clustering}.
 
Common measures for quantifying the (dis)similarity of generative models typically rely on process statistics estimated from observations using either parametric or nonparametric methods. Parametric methods 
yield good performance 
when the (parametric) model the estimation is based on matches the true (unknown) model well. Nonparametric methods typically outperform parametric ones in case of model mismatch \cite{stoica2005spectral}, a likely scenario in many practical applications. 
Existing random process clustering methods quantify the dissimilarity of observations using the Euclidean distance between estimated process model parameters \cite{corduas2008time, marti2016clustering}, cepstral coefficients \cite{kalpakis2001distance, boets2005clustering}, or normalized periodograms \cite{caiado2006periodogram}. Other methods rely on divergences (e.g., Kullback-Leibler divergence) between normalized periodograms \cite{kakizawa1998discrimination, vilar2004discriminant}, use the distributional distance \cite{gray2009probability} between processes \cite{ryabko2010clustering, khaleghi2012online, khaleghi2016consistent}, or the earth mover's distance between copulas of the processes \cite{marti2016clustering, marti2015optimal}. In all cases the resulting distances are fed into a standard clustering algorithm such as $k$-means or hierarchical clustering. Another line of work employs a Bayesian framework to infer the cluster assignments, e.g., according to a maximum a posteriori criterion \cite{xiong2004time}. While many of these approaches have proven effective in practice, corresponding analytical performance results are scarce. Moreover, existing analytical results are mostly concerned with the asymptotic regime where the observation length goes to infinity while the number of observations is fixed (see, e.g., \cite{kakizawa1998discrimination, vilar2004discriminant, corduas2008time, ryabko2010clustering, khaleghi2012online, borysov2014asymptotics}); the finite observation-length regime has attracted significantly less attention \cite{ryabko2010clustering, ryabko2013binary, khaleghi2016consistent, marti2016clustering}.

\paragraph*{Contributions}
We consider two process clustering algorithms that apply to nonparametric generative models and 
employ 
the $L^1$-distance between estimated power spectral densities (PSDs)
as dissimilarity measure. 
The first one, termed nearest neighbor process clustering (NNPC), relies on partitioning the $q$-nearest neighbor graph ($q$ is a parameter of the algorithm) of the observations via normalized spectral clustering 
and, to the best of our knowledge, has not been considered in the literature before. NNPC is inspired by the thresholding-based subspace clustering (TSC) algorithm \cite{heckel2014robust}, which 
clusters a set of (high-dimensional) data points into a union of low-dimensional subspaces without prior knowledge of the subspaces, their dimensions, and their orientations. 
The second algorithm, which will be referred to as KM, 
consists of a single $k$-means iteration with farthest point initialization \cite{katsavounidis1994new} and was first proposed in \cite{ryabko2010clustering}, albeit with a different dissimilarity measure.

Assuming real-valued stationary ergodic Gaussian processes with arbitrary (continuous) PSDs as generative models, we characterize the performance of NNPC and KM analytically for finite-length observations---potentially with missing entries---contaminated by independent additive real-valued white Gaussian noise. We find that both algorithms succeed with high probability even when the PSDs of the generative models exhibit significant overlap, all provided that the observation length is sufficiently large and the noise variance is sufficiently small. Our analytical results quantify the tradeoff between observation length, 
fraction of 
missing 
entries, noise variance, and distance between the (true) PSDs of the generative models.

Furthermore, we prove that treating the finite-length observations as vectors in Euclidean space and clustering them using the TSC algorithm \cite{heckel2014robust} (which inspired NNPC) results in performance strictly inferior to that obtained for NNPC. We argue that the underlying cause is to be found in TSC employing spherical distance as dissimilarity measure, thereby ignoring the stationary process structure of the observations. In a broader context this suggests that clustering observations of random processes using dissimilarity measures conceived with Euclidean geometry in mind, a popular ad-hoc approach in practice \cite{esling2012time}, can lead to highly suboptimal performance.

We evaluate the performance of NNPC and KM on synthetic and on real data, and find that NNPC outperforms state-of-the-art algorithms in human motion sequence clustering. Furthermore, NNPC and KM are shown to yield better clustering performance than single linkage and average linkage hierarchical clustering based on the $L^1$-distance between estimated PSDs. 
We also compare ($L^1$-based) NNPC and KM to their respective $L^2$ and $L^\infty$-cousins and find that the original variants consistently yield better or the same results.

\paragraph*{Relation to prior work} 
Numerical studies of time series clustering based on spectral clustering of the $q$-nearest neighbor graph using different dissimilarity measures (albeit not the $L^1$-distance, or, for that matter, other $L^p$-distances, between estimated PSDs) were reported in \cite{tucci2011analysis}. 
In \cite{ferreira2016time} time series clustering is formulated as a community detection problem in graphs, but no analytical performance results are provided. 
KM with distributional distance as dissimilarity measure was proven in \cite{ryabko2010clustering}---for more general (i.e., not necessarily Gaussian) generative models---to deliver correct clustering with probability approaching $1$ as the observation length goes to infinity. 
We note, however, that estimating the distributional distance is computationally more demanding than estimating the $L^1$-distance between PSDs.

\paragraph*{Notation} 
We use lowercase boldface letters to denote vectors, uppercase boldface letters to designate matrices, and the superscript $\transp{}$ stands for transposition. $v_i$ is the $i$th entry of the vector $\vv$. For the matrix $\mA$, $\mA_{i,j}$ denotes the entry in the $i$-th row and $j$-th column, $\mA_i$ its $i$-th row, $\norm[2 \to 2]{\mA}$ its spectral norm, $\norm[F]{\mA} \defeq (\sum_{i,j} \abs{\mA_{i,j}}^2)^{1/2}$ its Frobenius norm, and (for $\mA$ square) $\tr(\mA) = \sum_i \mA_{i,i}$ its trace. $\mI$ and $\mathbf 1$ stand for the identity matrix and the all ones matrix (the latter not necessarily square), 
respectively. For matrices $\mA$ and $\mB$ of identical dimensions, $\mA \circ \mB$ is the Hadamard product, i.e., $(\mA \circ \mB)_{i,j} = \mA_{i,j} \mB_{i,j}$. For the vector $\vb \in \{0,1\}^n$, 
we let $\mP_\vb \defeq \mathrm{diag}(b_1,\dots,b_n)$. The $i$-th element of a sequence $x$ is denoted by $x[i]$. For a positive integer $N$, $[N]$ stands for the set $\{1,2,\dots,N\}$. 
The (circular) convolution of $f,g \in L^2([0,1))$ is defined as $(f \ast g)(y) \defeq \int_0^1 f(x) \tilde g(y-x) \d x$, $y \in [0,1)$, where $\tilde g$ is the $1$-periodic extension of $g$. 
$\log$ refers to the natural logarithm. $\E{X}$ denotes the expectation of the random variable $X$ and the notation $Y \sim X$ indicates that the random variable $Y$ has the same distribution as $X$. We say that a subgraph $H$ of a graph $G$ is connected if every pair of nodes in $H$ can be joined by a path with nodes exclusively in $H$. A connected subgraph $H$ of $G$ is called a connected component of $G$ if there are no edges between $H$ and the remaining nodes in $G$. 

\section{Formal problem statement and algorithms} \label{sec:ProbAlgo}

We consider the following clustering problem. Given the unlabeled data set $\cX = \cX_1 \cup \ldots \cup \cX_L$ of cardinality $N$, where $\X_\l = \{ \xlh_i \}_{i=1}^{n_\l}$ contains noisy length-$M$ observations $\xlh_i$---possibly with missing entries---of the real-valued discrete-time stationary ergodic random process $\Xl[m]$, $m \in \mb Z$, corresponding to the $\l$-th generative model, find the partition $\cX_1,\dots,\cX_L$. The statistics of the generative models and of the noise processes, and the number of generative models, are all assumed unknown.

Both clustering algorithms considered in this paper are based on the following measure for the distance between pairs of processes. With the PSD of $\Xl$ denoted by $\sl(f)$, $f \in [0,1)$, we define the distance (dissimilarity) between the processes $\Xk$ and $\Xl$ as $d(\Xk, \Xl) \defeq \frac{1}{2} \int_0^1 \abss{\sk(f) - \sl(f)} \d f$. 
As argued below, for the algorithms to be meaningful, the different processes have to be of the same or at least of comparable power, which motivates the normalization $\int_0^1 \sl(f) \d f = 1$, $\l \in [L]$. Now, this implies that $d(\Xk,\Xl) \leq \frac{1}{2} \int_0^1 \abss{\sk(f)} \d f +   \frac{1}{2} \int_0^1 \abss{\sl(f)} \d f = \frac{1}{2} \int_0^1 \sk(f) \d f +   \frac{1}{2} \int_0^1 \sl(f) \d f = 1$, and hence $d(\Xk,\Xl) \in [0,1]$. The distance measure
$d(\Xk,\Xl)$ is close to $1$ when $\sk$ and $\sl$ are concentrated on disjoint frequency bands and close to $0$ when they exhibit similar support sets and shapes. 
In contrast, for general $L^p$-distances $d_{L^p}(\Xk, \Xl) \defeq  (\int_0^1 \abss{\sk(f) - \sl(f)}^p \d f)^\frac{1}{p}$, with $p >1$, it is easy to see that $\int_0^1 \sl(f) \d f = 1$, $\l \in [L]$, does not imply a uniform upper bound for $d_{L^p}(\Xk, \Xl)$. For example, $d_{L^\infty}(\Xk, \Xl)$ can become arbitrarily large if we set $\sk(f) = 1$, $f \in [0,1)$, and let $\sl$ have a sharp peak at some frequency $f_0 \in [0,1)$, while maintaining $\int_0^1 \sl(f) \d f = 1$.

We now present the NNPC and the KM algorithms.
Recall that NNPC is inspired by the TSC algorithm introduced in \cite{heckel2014robust}, and KM is obtained by replacing the distance measure in Algorithm 1 in \cite{ryabko2010clustering} by the distance measure $d$ defined above. In principle, NNPC and KM are applicable to general (real-valued) time series, in particular also to non-stationary random processes, but the definition of $d$ above is obviously motivated by stationarity. 

{\it \refstepcounter{AlgCount} \label{alg:TSC}
{\bf The NNPC algorithm. } Given a set $\cX$ of $N$ length-$M$ observations, the number of generative models $L$ (the estimation of $L$ from $\cX$ is discussed below), and the parameter $q$, carry out the following steps. \\
{\bf Step 1:} For every $\xh_i \in \cX$, estimate the PSD $\sh_i(f)$ via the Blackman-Tukey (BT) estimator according to
\vspace{-0.05cm}
\begin{align}
\sh_i(f) &\defeq \sum_{m=-M+1}^{M-1} g[m] \rh_i[m] e^{-\im 2 \pi f m}, \quad \text{where} \label{eq:BlackmanTukey} \\
\rh_i[m] &\defeq \frac{1}{M} \!\! \sum_{n = 0}^{M-\vert m \vert-1} \xh_i[n+m] \xh_i[n], \quad \vert m \vert \leq M-1, \nonumber
\end{align}
and $g[m]$, $m \in \mb Z$, is an even window function (i.e., $g[m] = g[-m]$) with $g[m] = 0$ for $|m| \geq M$, and with bounded non-negative discrete-time Fourier transform (DTFT). \\ 
{\bf Step 2:} For every $\xh_i \in \cX$, identify the set $\mc T_i \subset [N] \backslash \{ i \}$ of cardinality $q$ defined through
\begin{equation}
d(\xh_i,\xh_j) \leq d(\xh_i,\xh_v), \quad \text{for all} \; j \in \mc T_i \; \text{and all} \; v \notin \mc T_i, \nonumber
\end{equation}
where 
\begin{equation}
d(\xh_i, \xh_j) \defeq \frac{1}{2} \int_0^1 \abs{\sh_i(f) - \sh_j(f)} \d f. \label{eq:distest}
\end{equation}
{\bf Step 3:} Let $\vz_j \in \reals^N$ be the vector with $i$th entry $\exp(-2 \, d(\xh_i,\xh_j))$, if $i \in \mc T_j$, and $0$, if $i \notin \mc T_j$. \\
{\bf Step 4:} Construct the adjacency matrix $\mA$ according to $\mA = \mZ + \transp{\mZ}$, where $\mZ = [\vz_1 \, \dots \, \vz_N]$. \\
{\bf Step 5:} Apply normalized spectral clustering \cite{luxburg2007tutorial} to $(\mA,L)$.}

Step 2 of NNPC determines the $q$-nearest neighbors of every observation w.r.t. to the distance measure $d$. We henceforth denote the corresponding nearest neighbor graph with adjacency matrix $\mA$ constructed in Step 4 by $G$. 
The parameter $q$ determines the minimum degree of $G$. 
Choosing $q$ too small results in the observations stemming from a given generative model forming multiple connected components in $G$ and hence not being assigned to the same cluster in Step 5. This problem can be countered by taking $q$ larger, which, however, increases the chances of observations originating from different generative models being connected in $G$, thereby increasing the likelihood of incorrect cluster assignments. These tradeoffs are identical to those associated with the choice of the parameter $q$ in TSC \cite{heckel2014robust}. Note that spectral clustering is robust in the sense that it may deliver correct clustering even when $G$ contains edges connecting observations that originate from different generative models, as long as the corresponding edge weights are sufficiently small.

The number of generative models, $L$, may be estimated in Step 4 based on the adjacency matrix $\mA$ using the \emph{eigengap heuristic} \cite{luxburg2007tutorial} (note that $L$ is needed only in Step 5), which relies on the fact that the number of zero eigenvalues of the normalized Laplacian of $G$ equals the number of connected components in $G$.

{\it \refstepcounter{AlgCount} \label{alg:kmeans}
{\bf The KM algorithm \cite{ryabko2010clustering}. } Given a set $\cX$ of $N$ length-$M$ observations and the number of generative models $L$, carry out the following steps. \\
\setlength{\parindent}{1.7cm}
\setlength{\parskip}{0cm}
{\bf Step 1:} Initialize $c_1 \defeq 1$ and $\hat \cX_\l \defeq \{\}$, for all $\l \in [L]$. \\ 
{\bf Step 2:} For every $\xh_i \in \cX$, estimate the PSD $\sh_i(f)$ via the BT estimator \eqref{eq:BlackmanTukey}. \\
{\bf Step 3:} {\bf for $p = 2$ to $L$ do:}

$c_p \defeq \arg \max_{i \in [N]} \big(\min_{\l \in [p-1]} d(\xh_i,\xh_{c_\l})\big)$,\\
with $d$ as defined in \eqref{eq:distest}. \\ 
{\bf Step 4:} {\bf for $i = 1$ to $N$ do:}

$\l^\star \gets \arg \min_{\l \in [L]} \, d(\xh_i,\xh_{c_\l})$

$\hat \cX_{\l^\star} \gets \hat \cX_{\l^\star} \cup \{ \xh_i \}$
}

KM selects the cluster centers in Step 3 and determines the assignments of the observations to these cluster centers in Step 4. Specifically, the algorithm selects $\xh_1$ as the first cluster center and then recursively determines the remaining cluster centers by maximizing the minimum distance to the cluster centers  already chosen. In Step 4, it then assigns each observation to the closest cluster center (see Fig. \ref{fig:kmillustration}). 
Intuitively, KM recovers the correct cluster assignments if the clusters are separated well enough. 
In practice, performing additional $k$-means iterations by alternating between cluster center refinement (simply by taking the refined center to be the average of the observations assigned to it) and re-assignment of the data points to the refined cluster centers, can often improve performance.  
Numerical results on the effect of additional $k$-means iterations are provided in Sec.~\ref{sec:numres}. Our analytical results, however, all pertain to the case of a single $k$-means iteration per the definition of the KM algorithm above. 
Note that besides the number of clusters, $L$, KM does not have other parameters such as $q$ in NNPC.

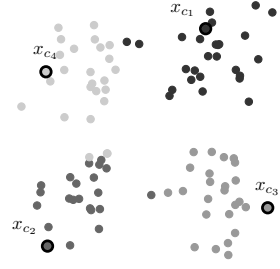
\begin{figure}[t]
\centering
\begin{tikzpicture}[scale = 0.7] 
\begin{axis}[clip mode=individual,axis equal,hide axis]
\addplot+[mark=*, color=black!80, only marks, mark options={solid}] table[x index=0,y index=1]{./data/fpillustration/X1.dat};
\addplot+[mark=*, color=black!60, only marks, mark options={solid}] table[x index=0,y index=1]{./data/fpillustration/X2.dat};
\addplot+[mark=*, color=black!40, only marks, mark options={solid}] table[x index=0,y index=1]{./data/fpillustration/X3.dat};
\addplot+[mark=*, color=black!20, only marks, mark options={solid}] table[x index=0,y index=1]{./data/fpillustration/X4.dat};

\node[label={120:$\color{black}x_{c_1}$},circle,inner sep=2pt,draw=black,line width=1.5pt] at (axis cs:0.7404,1.472) {};
\node[label={120:$\color{black}x_{c_2}$},circle,inner sep=2pt,draw=black,line width=1.5pt] at (axis cs:-1.564,-1.699) {};
\node[label={\color{black}$x_{c_3}$},circle,inner sep=2pt,draw=black,line width=1.5pt] at (axis cs:1.645,-1.14) {};
\node[label={\color{black}$x_{c_4}$},circle,inner sep=2pt,draw=black,line width=1.5pt] at (axis cs:-1.587,0.8433) {};

\end{axis}
\end{tikzpicture}
\vspace{-0.15cm}
\caption{\label{fig:kmillustration} 
Clustering of an example data set in $\reals^2$ determined by KM with farthest point intialization and based on Euclidean distance.}
\end{figure}

Both NNPC and KM are based on comparisons of distances between observations, and are, therefore, meaningful only if the underlying processes $\Xl$ are of comparable power $\int_0^1 \! \sl(f) \d f$. Indeed, 
when this is not the case, 
the distance between the observations 
is determined predominantly by the difference in power rather than the difference in PSD support sets and shapes. Note that the assumption of comparable power is not critical as,  in practice, we can normalize the observations.

The choice of the window function $g$ in \eqref{eq:BlackmanTukey} determines the bias-variance tradeoff of the BT estimator and through the distance estimates $d(\xh_i,\xh_j)$ ultimately the bias-variance tradeoff of NNPC and KM. For a discussion of window choice considerations for the BT estimator in a general context, we refer the reader to \cite[Sec. 2.6]{stoica2005spectral}. We only remark here that the variance of the BT estimator goes to $0$ as $M \to \infty$ under rather mild conditions on the process PSD and for $g \in \ell_1$ \cite[Appendix B4]{kay1988modern}; the statistical data model employed in this paper (and described in the next section) satisfies these conditions on the PSDs.

We finally briefly discuss computational aspects of NNPC and KM for $L \ll N$, the situation typically encountered in practice. 
The BT PSD estimates \eqref{eq:BlackmanTukey} can be computed efficiently using the FFT. NNPC is a spectral clustering algorithm and as such requires the $N(N-1)/2$ distances between all pairs of observations to construct $G$. NNPC furthermore needs to determine the $L$ eigenvectors corresponding to the $L$ smallest eigenvalues of the $N \times N$ normalized graph Laplacian, which requires $O(N^3)$ operations (without exploiting potentially present structural properties of the Laplacian such as, e.g., sparsity).  
Spectral clustering then performs standard $k$-means clustering on the rows of the resulting $N \times L$ matrix of eigenvectors. 
The computational complexity of NNPC therefore becomes challenging for large $N$. Several spectral clustering methods suitable for data sets of up to millions of observations are available in the literature, see, e.g., \cite{yan2009fast, li2011timeand, chen2011parallel}. KM, on the other hand, is computationally considerably less expensive, requiring only $O(NL^2)$ distance computations.

We finally note that both NNPC and KM along with the corresponding analytical performance guarantees presented in the next section can easily be generalized to stationary ergodic vector-processes $\mXl[m] \in \reals^n$, $m \in \mb Z$. Specifically, with the spectral density matrices $\mSl(f) \defeq \sum_{m=-\infty}^\infty \mb E\big[\mXl[m] \transp{(\smash{\mXl[0]})\!}\big] e^{-\im 2 \pi f m} \in \reals^{n \times n}$, $\l \in [L]$, one defines the distance measure
\begin{equation*}
d(\mXk, \mXl) = \sum_{u,v \in [n]} \int_0^1 \abs{\smash{\mSk_{u,v}(f) - \mSl_{u,v}(f)}} \d f
\end{equation*} and employs the BT estimator in \eqref{eq:BlackmanTukey} component-wise to estimate $\mSl(f)$. As this requires the computation of distances between all scalar random process components, evaluating $d(\mXk, \mXl)$ in the vector case incurs $n(n+1)/2$ (exploiting the symmetry of $\mSl(f)$) times the cost in the scalar case. 
All other steps of NNPC and KM remain unchanged and hence have the same computational complexity as in the scalar case. 
For simplicity of exposition, we focus on the scalar case throughout the paper.

\section{Analytical performance results} \label{sec:mainres}

We start by describing the statistical data model underlying our analytical performance results. Recall that both NNPC and KM are, in principle, applicable to general real-valued 
time series including non-stationary processes. 
The performance analysis conducted here applies, however, to stationary processes. In addition, we take into account additive noise and potentially missing entries. 
Specifically, we assume that the $\xlh_i$ are obtained as contiguous length-$M$ observations of $\Xlc[m] \defeq \El[m](\Xl[m] + \Wl[m]), m \in \mb Z$, where $\El$ is a Bernoulli process with i.i.d. entries according to $\PR{\El[m] = 1} = 1- \PR{\El[m] = 0} = \perr > 0$ (we henceforth refer to $p$ as sampling probability), 
$\Xl$ is zero-mean stationary Gaussian with PSD $\sl(f)$, and $\Wl$ is a zero-mean white Gaussian noise process with variance $\sigma^2$. The autocorrelation functions (ACFs) $\rl[m] \defeq \int_0^1 \sl(f) e^{\im 2 \pi f m} \d f$ of the $\Xl$ are assumed absolutely summable, i.e., $\sum_{m = -\infty}^\infty \abss{\rl[m]} < \infty$, $\l \in [L]$, which implies continuity of the $\sl(f)$ and thereby ergodicity of the corresponding processes $\Xl$ \cite{maruyama1949harmonic}. 
Moreover, we take the PSDs to be normalized according to $\int_0^1 \sl(f) \d f = 1$, $\l \in [L]$, and we let $B \defeq \max_{\l \in [L]} \sup_{f \in [0,1)} \sl(f)$. We further assume that $\El$, $\Xl$, and $\Wl$ are mutually independent. As a consequence, the noisy process $\Xlt[m] \defeq \Xl[m] + \Wl[m]$ and the Bernoulli process $\El$ are jointly stationary ergodic so that $\Xlc[m] = \El[m] \Xlt[m]$ is stationary ergodic by \cite[Prop. 3.36]{white2014asymptotic}. 
Furthermore, we denote the ACF of the noisy process $\Xlt[m]$ by $\rlt[m]$ and note that $\rlt[m] = \rl[m] + \sigma^2 \delta[m]$. It follows from $\Xlc[m] = \El[m] \Xlt[m]$ that $\rlc [m] = \errwin[m] \rlt[m]$, where $\errwin[m] \defeq p$ for $m = 0$, and $\errwin[m] \defeq p^2$, else. 
For each $\l$, the $\xlh_i$ may either stem from independent realizations of $\Xlc$ or correspond to different (possibly overlapping) length-$M$ segments of a given realization of $\Xlc$. In the latter case the $\xlh_i$ will not be statistically independent in general. This is, however, not an issue as statistical independence is not required in our analysis, neither across observations stemming from a given generative model nor across observations originating from different generative models.

Multiplication of $\Xlt$ by the Bernoulli process $\El$ models, e.g., a sampling device which  
acquires only every $(1/p)$-th sample on average. Moreover, in practice we could deliberately subsample in order to speed up the computation of the distances $d$ when the observation length $M$ is large. 
Specifically, with observation length $M$ and sampling probability $p$, we get $\approx (1-p)M$ entries of $\xlh_i$ that are set to $0$, which can be exploited when computing the BT estimates using the FFT~\cite{skinner1976pruning}.

Na{\"\i}vely applying the BT estimator to the $\xlh_i$ delivers PSD estimates that, owing to $\rlc [m] = \errwin[m] \rlt[m]$, can be severely biased compared to estimates that would be obtained from observations with no missing entries.  
Indeed, as $\rlc [0] = \perr \, \rlt[0]$ and $\rlc [m] = \perr^2 \rlt[m]$, for $m \neq 0$, for small $\perr$, $\errwin[m]$ assigns a much larger weight to  
lag $m=0$ than to the lags $m \neq 0$. To correct this bias, 
we assume in the remainder of the paper (in particular also in the analytical results below) that the BT estimates in \eqref{eq:BlackmanTukey} are computed for 
the window function $\gh[m] \defeq g[m]/ \errwin[m]$, $m \in \mb Z$, i.e., $g$ in Algorithms \ref{alg:TSC} and \ref{alg:kmeans} is replaced by $\gh$. 
While $\gh$ remains even and supported on $\{-M+1, \ldots, M-1\}$, BT PSD estimates based on $\gh$ are not guaranteed to be non-negative (in contrast to estimates based on $g$ directly \cite[Sec. 2.5.2]{stoica2005spectral}) as the DTFT of $\gh$ may not be non-negative. This is, however, not an issue as we consider distances between PSDs only and do not explicitly make use of the positivity property of PSDs. We note that bias correction requires knowledge of $\perr$, which can be obtained in practice simply by estimating the average number of non-zero entries in the $\xlh_i$. In addition, we will assume that $g[0] = 1$ and $g$ has a bounded DTFT $g(f)$, i.e., $0 \leq g(f) \leq A < \infty$, $f \in [0,1)$. 
An example of such a window function is the Bartlett window (see \eqref{eq:defbartlettwin}) used in the experiments in Sec. \ref{sec:numres}. 
Our performance results will be seen to depend on the maximum ACF moment $\mu_{\max} \defeq \max_{\l \in [L]} \mul$, where $\mul \defeq \sum_{m=-\infty}^\infty \abss{h[m]} \abss{\rl[m]}$ with
\begin{equation} \label{eq:weightedwindow}
h[m] \defeq \left\{ \begin{matrix*}[l] 1 - g[m](1 - \vert m \vert /M), & \text{for} \; \vert m \vert < M \\
1, & \text{otherwise.} \end{matrix*} \right.
\end{equation}

We are now ready to state our main results. For the NNPC algorithm, we provide a sufficient condition for the following \emph{no false connections (NFC) property} to hold. Recall that $G$ is the nearest neighbor graph with adjacency matrix $\mA$, as constructed in Step 4 of NNPC.
\begin{definition}[No False Connections Property] 
$G$ satisfies the no false connections property if, for all $\l \in [L]$, all nodes corresponding to $\cX_\l$ are connected exclusively to nodes corresponding to $\cX_\l$. 
\end{definition}
We henceforth say that ``NNPC succeeds'' if the NFC property is satisfied. Although the NFC property alone does not guarantee correct clustering, it was found to be a sensible performance measure for subspace clustering algorithms (see, e.g., \cite{elhamifar_sparse_2013, soltanolkotabi2014robust, heckel2014robust}). To ensure correct clustering one would additionally need the subgraph of $G$ corresponding to $\cX_\l$ to be connected, for each $\l \in [L]$ \cite[Prop.~4; Sec.~7]{luxburg2007tutorial}. 
Establishing conditions for this to hold appears to be difficult, at least for the statistical data model considered here.

\begin{thm} \label{thm:NNPC}
Let $\cX$ be generated according to the statistical data model described above and assume that $q \leq \min_{\l \in[L]} (n_\l - 1)$. Then, 
the clustering condition
\begin{align} 
&\min_{\substack{k, \l \in [L] \colon \\ k \neq \l}} d(\Xk, \Xl) \nonumber \\
& \qquad >  \frac{8 \sqrt{2} A (B \!+\! \sigma^2 \!+\! \sqrt{2} (1\!+\! p) (1 \!+\! \sigma^2))}{\perr^2} \sqrt{\frac{\log M}{M}} \!+\! 2 \mu_{\max} \label{eq:ClusCond}
\end{align}
guarantees that $G$ satisfies the NFC property with probability at least $1- 6N/M^2$.
\end{thm}

The condition $q \leq \min_{\l \in [L]} (n_\l - 1)$ is necessary for the NFC property to hold as choosing $q > \min_{\l \in [L]} (n_\l - 1)$ would force NNPC to select observations from $\cX \backslash \cX_\l$ for at least one of the data points $\xlh_i$. As the $n_\l$ are unknown in practice, one has to guess $q$ while taking into account the tradeoffs related to the choice of $q$ as discussed in Sec. \ref{sec:ProbAlgo}.

Our main result for KM comes with a performance guarantee that is stronger than the NFC property, namely it ensures correct clustering; accordingly, ``KM succeeds'' henceforth refers to KM delivering correct clustering. This stronger result is possible as KM does not entail a spectral clustering step and is therefore much easier to analyze. On the other hand, NNPC typically outperforms KM in practice, as seen in the numerical results in Sec.~\ref{sec:numres}.

\begin{thm} \label{thm:KM}
Let $\cX$ be generated according to the statistical data model described above. Then, under the clustering condition \eqref{eq:ClusCond}, the partition $\hat \cX_1,\dots, \hat \cX_L$ of $\cX$ inferred by KM corresponds to the true partition $\cX_1,\dots,\cX_L$ with probability at least $1- 6N/M^2$.
\end{thm}

The proofs of Theorems \ref{thm:NNPC} and \ref{thm:KM} are provided in Appendix \ref{sec:ProofMain}. We first note that the clustering condition \eqref{eq:ClusCond} depends on a few model 
parameters only and all constants involved are explicit. Furthermore, the condition is identical for NNPC and KM, although the performance guarantee we obtain for KM (namely correct clustering) is stronger than that for NNPC (namely the NFC property). This is a consequence of both proofs relying on the same ``separation condition'' (namely \eqref{eq:DistRelation} in Appendix \ref{sec:ProofMain}) 
and the clustering condition \eqref{eq:ClusCond} being sufficient for this separation condition to hold (see Appendix \ref{sec:ProofMain} for further details).

Theorems \ref{thm:NNPC} and \ref{thm:KM} essentially state that NNPC and KM succeed even when the PSDs $\sl$ of the $\Xl$ overlap significantly and the observations have missing entries and are contaminated by strong noise, all this provided that the observation length $M$ is sufficiently large and the window function $g$ is chosen  
to guarantee small $\mu_{\max}$. The clustering condition \eqref{eq:ClusCond} suggests (recall that it is sufficient only) a tradeoff between the amount of overlap of pairs of PSDs $\{\sk,\sl\}$ (through $\min_{k, \l \in [L] \colon k \neq \l} d(\Xk,\Xl)$), the observation length $M$, the sampling probability $\perr$, and the noise variance $\sigma^2$. 
It indicates, for example, that both algorithms tolerate shorter observation length $M$, more missing entries (i.e., smaller $p$), and stronger noise (i.e., larger $\sigma^2$) as the 
pairs $\{ \sk, \sl \}$, $k \neq \l$, overlap less and hence $\min_{k, \l \in [L] \colon k \neq \l} d(\Xk,\Xl)$ is larger. 
Keeping $\sigma^2$ and $p$ fixed, the first term on the RHS of \eqref{eq:ClusCond}, which accounts for the PSD estimation error owing to finite observation length $M$, vanishes as $M$ becomes large. 
Since $d(\Xk, \Xl) \in [0,1]$, we need $\mu_{\max} \ll 1$ to ensure that the clustering condition can be satisfied for finite $M$. To see how this can be accomplished, we consider $\rl$ of small effective support relative to $M$, i.e., $\rl[m] \approx 0$ for $m \geq M_0$ with $M_0 \ll M$, which is essentially equivalent to requiring that the $\sl$ be sufficiently smooth. We then choose $g$ such that $g[m] \approx 1$ for $m \leq M_0$ and note that this ensures $h[m] \approx 1-g[m] (1 - \vert m \vert /M) \approx 1-g[m] \approx 0$, for $m \leq M_0 \ll M$. Thanks to $\mul = \sum_{m=-\infty}^\infty \abss{h[m]} \abss{\rl[m]} \approx \sum_{m=-M_0}^{M_0} \abss{h[m]} \abss{\rl[m]} \ll 1$, we then get $\mu_{\max} \ll 1$.
The clustering condition \eqref{eq:ClusCond} can hence, indeed, be satisfied for finite $M$ if the $\rl$ have small effective support. 
Note that the choice of $g$ will affect the constant $A$ (recall that $0 \leq g(f) \leq A < \infty$, $f \in [0,1)$). Specifically, windows $g$ of larger effective support have larger corresponding $A$ in general.

To ensure high probability of success, we need to take $M \gg \sqrt{N}$, i.e., the observation length has to be large relative to the square root of the number of observations. 
We note that the results in Theorems \ref{thm:NNPC} and \ref{thm:KM} can easily be extended to colored noise processes, as long as the noise PSDs are identical for all $\l \in [L]$. 

We emphasize that the vast majority of analytical performance results for random process clustering available in the literature pertain to the asymptotic regime $M \to \infty$, with $N$ fixed. The findings in \cite{kakizawa1998discrimination, vilar2004discriminant} are closest in spirit to ours and show that pairs of observations stemming from different generative models can be discriminated consistently (in the statistical sense), for $M \to \infty$, via a PSD-based distance measure, provided that the PSDs of all pairs of generative models differ on a set of positive Lebesgue measure. 

Finally, we note that generalization of our analytical results to processes other than Gaussian such as, e.g., subgaussian processes, seems difficult as a version of the concentration inequality \cite[Lem. 1]{demanet2012matrix}, upon which the proofs of Theorems \ref{thm:NNPC} and \ref{thm:KM} rely, does not appear to be available for non-Gaussian random vectors with dependent entries (see \cite{adamczak2015note} for details). For i.i.d. subgaussian processes such an inequality was reported in \cite{hansonwright2013rudelson}; this is, however, not of interest here as i.i.d. processes have flat PSDs.

\section{Comparison with thresholding-based subspace clustering} 
For finite observation length $M$, the random process clustering problem considered here can also be cast as a classical subspace clustering problem 
simply by interpreting the observations $\xlh_i$ as vectors $\vxl_i \in \reals^M$. 
Numerical results, not reported here\footnote{but available at \url{http://www.nari.ee.ethz.ch/commth/research/}}, 
demonstrate, however, that this approach leads to NNPC significantly outperforming its subspace clustering cousin, 
the TSC algorithm \cite{heckel2014robust}. Our next result, Proposition \ref{pr:PropInnerProd} below, provides analytical underpinning for this observation. Before stating the formal result, we develop some intuition. 
To this end, we consider statistically independent observations and set $p=1$ (i.e., no missing entries). We then note that the clustering condition \eqref{eq:ClusCond} for NNPC ensures that (using \eqref{eq:DistUB} and \eqref{eq:DistLB} in \eqref{eq:DistRelation} together with \eqref{eq:PrFaUB} and \eqref{eq:PrFbUB}, cf. Appendices~\ref{sec:ProofMain} and~\ref{sec:PfDistCond})
\begin{equation}
\PR{d(\xkh_j, \xlh_i) \leq d(\xlh_v, \xlh_i)} < \frac{6}{M^2}, \label{eq:DissimProb}
\end{equation}
for $i \neq v$, all $j$, and $k \neq \l$. 
This guarantees that the probability of the NFC property being violated becomes small for $M$ large, in particular, even when the PSD pairs $\{ \sk, \sl \}$, $k \neq \l$, overlap substantially and $\mathrm{SNR} \defeq \rl[0]/\sigma^2 = 1/\sigma^2 < 1$, $\l \in [L]$. 
For TSC (which 
constructs the sets $\mc T_i$ such that $\abs{\innerprod{\vx_j}{\vx_i}} \geq \abs{\innerprod{\vx_v}{\vx_i}}$ for all $j \in \mc T_i$ and all $v \notin \mc T_i$) applied to $\{\vxl_i\}_{i \in n_\l, \l \in [L]}\vspace{0.9mm}$ the probability corresponding to the LHS of \eqref{eq:DissimProb} is $\PR{\smash{\abs{\innerprod{\smash{\vxk_j}}{\smash{\vxl_i}}} \geq \abs{\innerprod{\smash{\vxl_v}}{\smash{\vxl_i}}}}}$. The next proposition establishes that 
this probability remains bounded away from $0$ even when $M$ grows large, unless the observations are noiseless and all the PSD pairs $\{\sk, \sl \}$, $k \neq \l$, are supported on essentially disjoint frequency bands. 
These conditions are, however, hardly encountered in practice, and the corresponding clustering problem can be considered easy. The superior performance of NNPC as compared to TSC stems from the TSC similarity measure not exploiting the stationarity of the generative models. 
We proceed to the formal statement. 

\begin{proposition} \label{pr:PropInnerProd} 
Let $\xlh_i$ be a contiguous length-$M$ observation of $\Xlt$ (note that we consider the case $p=1$). Assume that the $\xlh_i$ are independent across $\l \in [L]$ and $i \in [n_\l]$. Denote the vectors containing the elements of the $\xlh_i$ by $\vxl_i \in \reals^M$ and the corresponding covariance matrices by $\Rlt \defeq \Rl + \sigma^2 \mI$, with $\Rl_{v,w} = \rl[w-v] = \rl[v-w]$, $\l \in [L]$. Then, for $k \neq \l$ and $v \neq i$, we have
\begin{align}
&\PR{\abs{\innerprod{\vxk_j}{\vxl_i}} \geq \abs{\innerprod{\vxl_v}{\vxl_i}}} \nonumber \\
& \qquad \qquad \qquad
\geq \frac{1}{5 \pi} \atan \! \left( \frac{\sqrt{\tr \! \left(\Rkt \Rlt \right)}}{5 \sqrt{3} \sqrt{\tr \! \left(\Rlt \Rlt \right)}} \right).\label{eq:PropInnerProd}
\end{align}
\end{proposition}

Proof: See Appendix \ref{sec:TSCpropproof}.

\begin{remark}
Note that, in contrast to Theorems \ref{thm:NNPC} and \ref{thm:KM}, Proposition \ref{pr:PropInnerProd} assumes the observations to be statistically independent. This assumption turns out to be critical in the proof of Proposition~\ref{pr:PropInnerProd}.
\end{remark}

We next show, as announced, that the RHS of \eqref{eq:PropInnerProd} remains strictly positive even when $M$ grows large, unless the observations are noiseless and all pairs of PSDs have essentially disjoint support. 
To this end, we examine the 
behavior of $(1/M)\tr (\Rkt \Rlt)$, $k \neq \l$, and $(1/M)\tr (\Rlt \Rlt)$ 
(the motivation for the normalization by $M$ will become clear later). First, note that $(1/M)\tr (\Rlt \Rlt) \allowbreak < \infty$ as $\sum_{m = - \infty}^\infty \left({\tilde r}^{(\l)} [m]\right)^2< \infty$ by virtue of $\rlt = \rl[m] + \sigma^2 \delta[m] \in \ell_1$, which, in turn, follows from the assumption $\rl \in \ell_1$. The probability in \eqref{eq:PropInnerProd} is hence bounded away from $0$ 
unless $(1/M)\tr (\Rkt \Rlt)\allowbreak\approx 0$. It therefore remains to identify conditions for $(1/M)\tr (\Rkt \Rlt) \approx 0$ to hold. To this end, we 
note that
\begin{align}
\frac{1}{M} \tr \! \left( \Rkt \Rlt \right) &= \frac{1}{M} \sum_{m=0}^{M-1} \sum_{n =0}^{M-1} \rkt[n-m] \rlt[n-m] \nonumber \\
&=\sum_{m \in \mc M} \left(1 - \frac{\vert m \vert}{M} \right) \rkt[m] \rlt[m] \label{eq:toepwin} \\ 
&= \int_0^1 (w \ast \skt)(f) \slt(f) \d f \label{eq:AutocorrParseval},
\end{align}
where \eqref{eq:toepwin} is due to the Toeplitz structure and the symmetry of $\Rkt$ and $\Rlt$, \eqref{eq:AutocorrParseval} is by Parseval's Theorem, $\mc M \defeq \{ -M +1, -M + 2, \dots, M - 1\}$, and $w(f) \defeq \sum_{m \in \mc M} (1 - \vert m \vert / M) e^{-\im 2 \pi f m} = \sin^2(\pi f M)/(M \sin^2(\pi f))$. As $w(f)$ is strictly positive on the interval $[0,1)$ (apart from its zeros which are supported on a set of measure $0$) and the $\slt(f)$, $\l \in [L]$, are non-negative, 
\eqref{eq:AutocorrParseval} is bounded away from $0$ for finite $M$. As $M$ grows large, $w(f)$ approaches the Dirac delta distribution, i.e., the ``leakage'' induced by $w$ becomes small and we have \eqref{eq:AutocorrParseval}$\;\approx\int_0^1 \skt(f) \slt(f) \d f$. This integral vanishes for all $k \neq \l$ if and only if $\sigma^2 = 0$ (recall that $\slt(f) = \sl(f) + \sigma^2$, $\l \in [L]$) and all pairs $\{\sk, \sl \}$, $k \neq \l$, are supported on essentially disjoint frequency bands. This establishes the claim made above and concludes the argument.

\urldef{\code}\url{http://www.nari.ee.ethz.ch/commth/research/}
\vspace{-0.1cm}
\section[Numerical results]{Numerical results\footnote{\label{note1}Matlab code available at {\code}}}
\label{sec:numres}

We present numerical results for NNPC and KM on synthetic and on real data. 
In addition, we report results for KM followed by 100 $k$-means iterations (see the discussion in Sec. \ref{sec:ProbAlgo}); this variant of KM will be referred to as iterated $k$-means (KMit).
Furthermore, we compare NNPC, KM, and KMit with single linkage (SL), average linkage (AL), and complete linkage (CL) hierarchical clustering \cite[Sec. 14.3.12]{friedman2009elements}, all based on the $L^1$-distance measure \eqref{eq:distest}. 
We also investigate variants of NNPC, KM, and KMit with the $L^1$-distance measure replaced by $d_{L^2}(\xh_i, \xh_j) \defeq (\int_0^1 \abs{\sh_i(f) - \sh_j(f)}^2 \d f)^\frac{1}{2}$, and variants of NNPC and KM with the $L^1$-distance measure replaced by $d_{L^\infty}(\xh_i, \xh_j) \defeq \sup_{f\in[0,1)} \abs{\sh_i(f) - \sh_j(f)}$ (we do not consider KMit here as 
$d_{L^\infty}$-based $k$-means iterations do not seem sensible). NNPC and KM were implemented strictly according to the corresponding algorithm descriptions in Sec.~\ref{sec:ProbAlgo}. For SL, AL, and CL, we use the functions built into Matlab. 
Throughout, performance is measured in terms of the clustering error (CE), i.e., the fraction of misclustered data points, defined as
\begin{equation}
\text{CE}(\hat \vc, \vc) = \min_\pi \left( 1 - \frac{1}{N} \sum_{i=1}^N 1_{\{\pi(\hat c_i) = \pi(c_i)\}} \right), \nonumber
\end{equation}
where $\vc \in [L]^N$ and $\hat \vc \in [L]^N$ are the true and the estimated assignments, respectively, and the minimum is taken over all permutations $\pi \colon [L] \to [L]$. 
We report running times (excluding time for loading the data) obtained on a MacBook Pro with a 2.5 GHz Intel Core i7 CPU with 16 GB RAM.

\subsection{Synthetic data} \label{sec:numsynthdata}

We investigate the tradeoff between the minimum distance $\min_{k, \l \in [L] \colon k \neq \l} d(\Xk,\Xl)$, the observation length $M$, the sampling probability $p$, and the noise variance $\sigma^2$ as indicated by the clustering condition \eqref{eq:ClusCond}. Recall that the clustering condition is only sufficient (and for NNPC guarantees the NFC property only). It is therefore unclear a priori to what extent the CE, indeed, follows the behavior indicated by the clustering condition. 

We consider $L=2$ second-order AR generative processes with PSDs of the form 
\begin{equation} \label{eq:resonator}
s_{a, \nu} (f) = \frac{b^2(a, \nu) }{\abs{1 - 2 a \cos ( \nu ) e^{\im 2\pi f} + a^2 e^{\im 4\pi f} }^2},
\end{equation}
where $\nu \in [0, \pi]$, 
$a \in (0,1)$, and $b^2(a, \nu) = 1/(\int_0^1 1/ \abs{1 - 2 a \cos ( \nu ) e^{\im 2\pi f} + a^2 e^{\im 4\pi f} }^2 \d f)$ ensures that $\int_0^1 s_{a, \nu} (f) \d f = 1$. Fig. \ref{fig:psdexamples} shows examples of $s_{a, \nu}(f)$ for different choices of $a$ and $\nu$. In the ensuing experiments, we set $\powpar{s}{1}(f) = s_{0.6,0.7\pi}(f)$ and $\powpar{s}{2}(f) = s_{0.6,\nu_2}(f)$, where $\nu_2$ is variable and controls the locations of the peaks of $\powpar{s}{2}$ and thereby the distance $d(\powpar{X}{1},\powpar{X}{2})$. Indeed, varying $\nu_2$ shifts the locations of the peaks of $\powpar{s}{2}$ while essentially maintaining its shape. 
For the BT PSD estimator, we use a Bartlett window of length $W$ defined as
\begin{equation}
g^B_W[m] \defeq \left\{ \begin{matrix*}[l] 1 - \vert m \vert / \lfloor W/2\rfloor, & \text{for} \; \vert m \vert \leq \lfloor W/2 \rfloor \label{eq:defbartlettwin} \\
0, & \text{otherwise,} \end{matrix*} \right.
\end{equation}
and we set $W = 101$. Note that $g^B_W$ satisfies the assumptions made about $g$ in Sec. \ref{sec:mainres}. The number of generative models $L=2$ is assumed known throughout. The performance of NNPC is found (corresponding results are not shown here) to be rather insensitive to the choice of the parameter $q$ as long as $10 \leq q \leq 25$; we set $q = 10$. For a given quadruple $(\nu_2,M,\sigma,p)$, a realization of the data set $\cX$ is obtained by sampling $n = 25$ independent observations from $\powpar{\check X}{1}$ and $\powpar{\check X}{2}$ each, and the CE is estimated by averaging over $10$ such independent realizations of $\cX$. We do not normalize the BT PSD estimates to unit power.

Fig. \ref{fig:phasediag} shows that NNPC, KM, and KMit all 
exhibit roughly the same qualitative behavior as a function of $d(\powpar{X}{1},\powpar{X}{2})$, $M$, $1/p$, and $\sigma$. In particular, for large enough $d(\powpar{X}{1},\powpar{X}{2})$ all three algorithms yield a CE close to $0$ even when $\sigma^2$ exceeds the signal power (i.e., when $\mathrm{SNR} < 1$), when the observations have missing entries ($p < 1$), and when $M$ is small. All three algorithms tolerate more noise and more missing entries as the observation length increases. These numerical results are in line with the \emph{qualitative} tradeoff indicated by the (sufficient) clustering condition \eqref{eq:ClusCond}. The numerical constants in (4) are, however, too big for the clustering condition (4) to be sharp. NNPC consistently achieves the lowest CE, followed by KMit, and KM. The performance advantage of NNPC over KM and KMit can be attributed to the spectral clustering step, which leads to increased robustness to noise and missing entries. Finally, we note that KMit often yields a significantly lower CE than KM.

The results in Fig. \ref{fig:phasediag-hierarchical} indicate that the \emph{qualitative} dependence of the CE on $d(\powpar{X}{1},\powpar{X}{2})$, $M$, $1/p$, and $\sigma$ for SL, AL, and CL is essentially identical to that for NNPC, KM, and KMit. For large $\sigma$ and small $d(\powpar{X}{1},\powpar{X}{2})$, $M$, or $p$, SL and AL lead, however, to a significantly larger CE than NNPC, KM, and KMit. The CE for CL is comparable to, but slightly larger than, that of KMit and significantly larger than that of NNPC. 

Comparing the CE for NNPC, KM, and KMit in Fig.~\ref{fig:phasediag} with that obtained for their $d_{L^2}$ and $d_{L^\infty}$-cousins in Figs.~ \ref{fig:phasediag-l2} and \ref{fig:phasediag-linf}, respectively, we note that, for all values of $d(\powpar{X}{1},\powpar{X}{2})$, $M$, $\sigma$, and $\perr$ the $d_{L^2}$-based variants of NNPC, KM, and KMit and the $d_{L^\infty}$-based variants of NNPC and KM yield the same or larger CE than the respective original variants. 
This justifies usage of the $L^1$-based distance measure \eqref{eq:distest} 
also from a practical point of view. 
Finally, we note that normalizing the model PSDs \eqref{eq:resonator} according to $(\int_0^1 s_{a, \nu}^2 (f) \d f)^\frac{1}{2} = 1$ for $d_{L^2}$ and $\sup_{f \in [0,1)} s_{a, \nu} (f) = 1$ for $d_{L^\infty}$ does not have a noticeable impact on clustering performance.

\begin{figure}[h!]
    \centering
    \begin{tikzpicture}[scale=0.8] 
    
    \begin{axis}[name=plot1,
        	xlabel={\footnotesize $f$},
    	width=0.9\columnwidth,
	height = 0.45\columnwidth,
	ymin = 0,
	ymax = 5,
	legend entries={ {$s_{0.6,0.4\pi}$}, {$s_{0.8,0.4\pi}$},{$s_{0.6,0.7\pi}$},{$s_{0.8,0.7\pi}$},},
	xticklabel style={font=\footnotesize,},
	yticklabel style={font=\footnotesize,},
    	legend style={
                    font=\tiny,}
         ]
	{
		\addplot +[mark=none,solid,black] table[x index=0,y index=1]{./data/phasediag/psds.dat};
		\addplot +[mark=none,dashdotted,black] table[x index=0,y index=2]{./data/phasediag/psds.dat};
		\addplot +[mark=none,dashed,black] table[x index=0,y index=3]{./data/phasediag/psds.dat};
		\addplot +[mark=none,dotted,black] table[x index=0,y index=4]{./data/phasediag/psds.dat};
		}
	\end{axis}
    \end{tikzpicture}
    \vspace{-0.2cm}
    \caption{\label{fig:psdexamples} Example PSDs of the form \eqref{eq:resonator}.}
\end{figure}
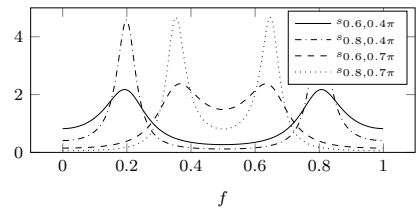

\newcommand{\labelsize}{\scriptsize}

\begin{figure}[h]
\centering
\begin{tikzpicture}[scale = 1]
\begin{groupplot}[group style={group size=3 by 4,horizontal sep=0.08\columnwidth,vertical sep=0.12\columnwidth,xlabels at=edge bottom, ylabels at=edge left},
x label style={at={(axis description cs:0.5,-0.35)},anchor=north},
width=0.34\columnwidth, /tikz/font=\footnotesize, colormap/blackwhite, view={0}{90}, point meta min=0.0, point meta max=0.5, minor tick num=4] 
\nextgroupplot[title = NNPC, xlabel={\labelsize $\sigma$}, ylabel={\labelsize $d(X^{(1)},X^{(2)})$}]
\addplot3[surf, shader=flat] file {./data/phasediag/cesnnpc-distsigman.dat};
\nextgroupplot[title = KM, xlabel={\labelsize $\sigma$}] 
\addplot3[surf, shader=flat] file {./data/phasediag/ceskm-distsigman.dat};
\nextgroupplot[title = KMit, colorbar, colorbar style={
at={(0.712\columnwidth, 0.115\columnwidth)},anchor=north west}, xlabel={\labelsize $\sigma$}] 
\addplot3[surf, shader=flat] file {./data/phasediag/ceskmfp-distsigman.dat};

\nextgroupplot[xlabel={\labelsize $M$}, ylabel={\labelsize $d(X^{(1)},X^{(2)})$}, scaled x ticks=base 10:-3, every x tick scale label/.style={at={(xticklabel* cs:1.15,0.25cm)}, anchor=near xticklabel}] 
\addplot3[surf, shader=flat] file {./data/phasediag/cesnnpc-distobslen.dat};
\nextgroupplot[xlabel={\labelsize $M$}, scaled x ticks=base 10:-3, every x tick scale label/.style={at={(xticklabel* cs:1.15,0.25cm)}, anchor=near xticklabel}] 
\addplot3[surf, shader=flat] file {./data/phasediag/ceskm-distobslen.dat};
\nextgroupplot[colorbar, xlabel={\labelsize $M$}, scaled x ticks=base 10:-3, every x tick scale label/.style={at={(xticklabel* cs:1.15,0.25cm)}, anchor=near xticklabel}] 
\addplot3[surf, shader=flat] file {./data/phasediag/ceskmfp-distobslen.dat};

\nextgroupplot[xlabel={\labelsize $M$}, ylabel={\labelsize $\sigma$}, scaled x ticks=base 10:-3, every x tick scale label/.style={at={(xticklabel* cs:1.15,0.25cm)}, anchor=near xticklabel}] 
\addplot3[surf, shader=flat] file {./data/phasediag/cesnnpc-obslensigman.dat};
\nextgroupplot[xlabel={\labelsize $M$}, scaled x ticks=base 10:-3, every x tick scale label/.style={at={(xticklabel* cs:1.15,0.25cm)}, anchor=near xticklabel}] 
\addplot3[surf, shader=flat] file {./data/phasediag/ceskm-obslensigman.dat};
\nextgroupplot[colorbar, xlabel={\labelsize $M$}, scaled x ticks=base 10:-3, every x tick scale label/.style={at={(xticklabel* cs:1.15,0.25cm)}, anchor=near xticklabel}] 
\addplot3[surf, shader=flat] file {./data/phasediag/ceskmfp-obslensigman.dat};

\nextgroupplot[xlabel={\labelsize $M$}, ylabel={\labelsize $1/\perr$}, scaled x ticks=base 10:-3, every x tick scale label/.style={at={(xticklabel* cs:1.15,0.25cm)}, anchor=near xticklabel}] 
\addplot3[surf, shader=flat] file {./data/phasediag/cesnnpc-obslenp.dat};
\nextgroupplot[xlabel={\labelsize $M$}, scaled x ticks=base 10:-3, every x tick scale label/.style={at={(xticklabel* cs:1.15,0.25cm)}, anchor=near xticklabel}] 
\addplot3[surf, shader=flat] file {./data/phasediag/ceskm-obslenp.dat};
\nextgroupplot[colorbar, xlabel={\labelsize $M$}, scaled x ticks=base 10:-3, every x tick scale label/.style={at={(xticklabel* cs:1.15,0.25cm)}, anchor=near xticklabel}]  
\addplot3[surf, shader=flat] file {./data/phasediag/ceskmfp-obslenp.dat}; 

\end{groupplot}
\end{tikzpicture}
\caption{\label{fig:phasediag} Results of the synthetic data experiment. First row: CE as a function of $\sigma$ and $d(\powpar{X}{1},\powpar{X}{2})$ for $M = 400$ and $p = 1$. Second row: CE as a function of $M$ and $d(\powpar{X}{1},\powpar{X}{2})$ for $\sigma = 0.5$ and $p =1$. Third row: CE as a function of $M$ and $\sigma$ for $\nu_2 = 0.62 \pi$ ($d(\powpar{X}{1},\powpar{X}{2}) \approx 0.2$) and $p = 1$. Bottom row: CE as a function of $M$ and $1/p$ for $\nu_2 = 0.62 \pi$ and $\sigma = 0.5$.}
\vspace{0.6cm}

\begin{tikzpicture}[scale = 1]
\begin{groupplot}[group style={group size=3 by 4,horizontal sep=0.08\columnwidth,vertical sep=0.12\columnwidth,xlabels at=edge bottom, ylabels at=edge left},
x label style={at={(axis description cs:0.5,-0.35)},anchor=north},
width=0.34\columnwidth, /tikz/font=\footnotesize, colormap/blackwhite, view={0}{90}, point meta min=0.0, point meta max=0.5, minor tick num=4] 
\nextgroupplot[title = NNPC($d_{L^2}$), xlabel={\labelsize $\sigma$}, ylabel={\labelsize $d(X^{(1)},X^{(2)})$}]
\addplot3[surf, shader=flat] file {./data/phasediag/cesnnpcl2-distsigman.dat};
\nextgroupplot[title = KM($d_{L^2}$), xlabel={\labelsize $\sigma$}] 
\addplot3[surf, shader=flat] file {./data/phasediag/ceskml2-distsigman.dat};
\nextgroupplot[title = KMit($d_{L^2}$), colorbar, colorbar style={
at={(0.712\columnwidth, 0.115\columnwidth)},anchor=north west}, xlabel={\labelsize $\sigma$}] 
\addplot3[surf, shader=flat] file {./data/phasediag/ceskmfpl2-distsigman.dat};

\nextgroupplot[xlabel={\labelsize $M$}, ylabel={\labelsize $d(X^{(1)},X^{(2)})$}, scaled x ticks=base 10:-3, every x tick scale label/.style={at={(xticklabel* cs:1.15,0.25cm)}, anchor=near xticklabel}] 
\addplot3[surf, shader=flat] file {./data/phasediag/cesnnpcl2-distobslen.dat};
\nextgroupplot[xlabel={\labelsize $M$}, scaled x ticks=base 10:-3, every x tick scale label/.style={at={(xticklabel* cs:1.15,0.25cm)}, anchor=near xticklabel}] 
\addplot3[surf, shader=flat] file {./data/phasediag/ceskml2-distobslen.dat};
\nextgroupplot[colorbar, xlabel={\labelsize $M$}, scaled x ticks=base 10:-3, every x tick scale label/.style={at={(xticklabel* cs:1.15,0.25cm)}, anchor=near xticklabel}] 
\addplot3[surf, shader=flat] file {./data/phasediag/ceskmfpl2-distobslen.dat};

\nextgroupplot[xlabel={\labelsize $M$}, ylabel={\labelsize $\sigma$}, scaled x ticks=base 10:-3, every x tick scale label/.style={at={(xticklabel* cs:1.15,0.25cm)}, anchor=near xticklabel}] 
\addplot3[surf, shader=flat] file {./data/phasediag/cesnnpcl2-obslensigman.dat};
\nextgroupplot[xlabel={\labelsize $M$}, scaled x ticks=base 10:-3, every x tick scale label/.style={at={(xticklabel* cs:1.15,0.25cm)}, anchor=near xticklabel}] 
\addplot3[surf, shader=flat] file {./data/phasediag/ceskml2-obslensigman.dat};
\nextgroupplot[colorbar, xlabel={\labelsize $M$}, scaled x ticks=base 10:-3, every x tick scale label/.style={at={(xticklabel* cs:1.15,0.25cm)}, anchor=near xticklabel}] 
\addplot3[surf, shader=flat] file {./data/phasediag/ceskmfpl2-obslensigman.dat};

\nextgroupplot[xlabel={\labelsize $M$}, ylabel={\labelsize $1/\perr$}, scaled x ticks=base 10:-3, every x tick scale label/.style={at={(xticklabel* cs:1.15,0.25cm)}, anchor=near xticklabel}] 
\addplot3[surf, shader=flat] file {./data/phasediag/cesnnpcl2-obslenp.dat};
\nextgroupplot[xlabel={\labelsize $M$}, scaled x ticks=base 10:-3, every x tick scale label/.style={at={(xticklabel* cs:1.15,0.25cm)}, anchor=near xticklabel}] 
\addplot3[surf, shader=flat] file {./data/phasediag/ceskml2-obslenp.dat};
\nextgroupplot[colorbar, xlabel={\labelsize $M$}, scaled x ticks=base 10:-3, every x tick scale label/.style={at={(xticklabel* cs:1.15,0.25cm)}, anchor=near xticklabel}]  
\addplot3[surf, shader=flat] file {./data/phasediag/ceskmfpl2-obslenp.dat}; 

\end{groupplot}
\end{tikzpicture}
\caption{\label{fig:phasediag-l2} CE as a function of $d(\powpar{X}{1},\powpar{X}{2})$, $M$, $\sigma$, and $\perr$ for variants of NNPC, KM, and KMit based on $d_{L^2}$, using the same values as in the setup in Fig. \ref{fig:phasediag} for the model parameters that are not varied.}
\vspace{-0.2cm}
\end{figure}

\begin{figure}[t]
\centering
\begin{tikzpicture}[scale = 1]
\begin{groupplot}[group style={group size=3 by 4,horizontal sep=0.08\columnwidth,vertical sep=0.12\columnwidth,xlabels at=edge bottom, ylabels at=edge left},
x label style={at={(axis description cs:0.5,-0.35)},anchor=north},
width=0.34\columnwidth, /tikz/font=\footnotesize, colormap/blackwhite, view={0}{90}, point meta min=0.0, point meta max=0.5, minor tick num=4]
\nextgroupplot[title = SL, xlabel={\labelsize $\sigma$}, ylabel={\labelsize $d(X^{(1)},X^{(2)})$}]
\addplot3[surf, shader=flat] file {./data/phasediag/cessl-distsigman.dat};
\nextgroupplot[title = AL, xlabel={\labelsize $\sigma$}] 
\addplot3[surf, shader=flat] file {./data/phasediag/cesal-distsigman.dat};
\nextgroupplot[title = CL, colorbar, colorbar style={
at={(0.712\columnwidth, 0.115\columnwidth)},anchor=north west}, xlabel={\labelsize $\sigma$}] 
\addplot3[surf, shader=flat] file {./data/phasediag/cescl-distsigman.dat};

\nextgroupplot[xlabel={\labelsize $M$}, ylabel={\labelsize $d(X^{(1)},X^{(2)})$}, scaled x ticks=base 10:-3, every x tick scale label/.style={at={(xticklabel* cs:1.15,0.25cm)}, anchor=near xticklabel}] 
\addplot3[surf, shader=flat] file {./data/phasediag/cessl-distobslen.dat};
\nextgroupplot[xlabel={\labelsize $M$}, scaled x ticks=base 10:-3, every x tick scale label/.style={at={(xticklabel* cs:1.15,0.25cm)}, anchor=near xticklabel}] 
\addplot3[surf, shader=flat] file {./data/phasediag/cesal-distobslen.dat};
\nextgroupplot[colorbar, xlabel={\labelsize $M$}, scaled x ticks=base 10:-3, every x tick scale label/.style={at={(xticklabel* cs:1.15,0.25cm)}, anchor=near xticklabel}] 
\addplot3[surf, shader=flat] file {./data/phasediag/cescl-distobslen.dat};

\nextgroupplot[xlabel={\labelsize $M$}, ylabel={\labelsize $\sigma$}, scaled x ticks=base 10:-3, every x tick scale label/.style={at={(xticklabel* cs:1.15,0.25cm)}, anchor=near xticklabel}] 
\addplot3[surf, shader=flat] file {./data/phasediag/cessl-obslensigman.dat};
\nextgroupplot[xlabel={\labelsize $M$}, scaled x ticks=base 10:-3, every x tick scale label/.style={at={(xticklabel* cs:1.15,0.25cm)}, anchor=near xticklabel}] 
\addplot3[surf, shader=flat] file {./data/phasediag/cesal-obslensigman.dat};
\nextgroupplot[colorbar, xlabel={\labelsize $M$}, scaled x ticks=base 10:-3, every x tick scale label/.style={at={(xticklabel* cs:1.15,0.25cm)}, anchor=near xticklabel}] 
\addplot3[surf, shader=flat] file {./data/phasediag/cescl-obslensigman.dat};

\nextgroupplot[xlabel={\labelsize $M$}, ylabel={\labelsize $1/p$}, scaled x ticks=base 10:-3, every x tick scale label/.style={at={(xticklabel* cs:1.15,0.25cm)}, anchor=near xticklabel}] 
\addplot3[surf, shader=flat] file {./data/phasediag/cessl-obslenp.dat};
\nextgroupplot[xlabel={\labelsize $M$}, scaled x ticks=base 10:-3, every x tick scale label/.style={at={(xticklabel* cs:1.15,0.25cm)}, anchor=near xticklabel}] 
\addplot3[surf, shader=flat] file {./data/phasediag/cesal-obslenp.dat};
\nextgroupplot[colorbar, xlabel={\labelsize $M$}, scaled x ticks=base 10:-3, every x tick scale label/.style={at={(xticklabel* cs:1.15,0.25cm)}, anchor=near xticklabel}] 
\addplot3[surf, shader=flat] file {./data/phasediag/cescl-obslenp.dat};

\end{groupplot}
\end{tikzpicture}
\caption{\label{fig:phasediag-hierarchical} CE for single linkage, average linkage, and complete linkage hierarchical clustering as a function of $d(\powpar{X}{1},\powpar{X}{2})$, $M$, $\sigma$, and $p$, using the same values as in the setup in Fig. \ref{fig:phasediag} for the model parameters that are not varied.}
\vspace{1.04cm}

\begin{tikzpicture}[scale = 1]
\begin{groupplot}[group style={group size=2 by 4,horizontal sep=0.08\columnwidth,vertical sep=0.12\columnwidth,xlabels at=edge bottom, ylabels at=edge left},
x label style={at={(axis description cs:0.5,-0.35)},anchor=north},
width=0.34\columnwidth, /tikz/font=\footnotesize, colormap/blackwhite, view={0}{90}, point meta min=0.0, point meta max=0.5, minor tick num=4] 
\nextgroupplot[title = NNPC($d_{L^\infty}$), xlabel={\labelsize $\sigma$}, ylabel={\labelsize $d(X^{(1)},X^{(2)})$}]
\addplot3[surf, shader=flat] file {./data/phasediag/cesnnpclinf-distsigman.dat};
\nextgroupplot[colorbar, colorbar style={
at={(0.47\columnwidth, 0.115\columnwidth)},anchor=north west}, title = KM($d_{L^\infty}$), xlabel={\labelsize $\sigma$}] 
\addplot3[surf, shader=flat] file {./data/phasediag/ceskmlinf-distsigman.dat};

\nextgroupplot[xlabel={\labelsize $M$}, ylabel={\labelsize $d(X^{(1)},X^{(2)})$}, scaled x ticks=base 10:-3, every x tick scale label/.style={at={(xticklabel* cs:1.15,0.25cm)}, anchor=near xticklabel}] 
\addplot3[surf, shader=flat] file {./data/phasediag/cesnnpclinf-distobslen.dat};
\nextgroupplot[colorbar, xlabel={\labelsize $M$}, scaled x ticks=base 10:-3, every x tick scale label/.style={at={(xticklabel* cs:1.15,0.25cm)}, anchor=near xticklabel}] 
\addplot3[surf, shader=flat] file {./data/phasediag/ceskmlinf-distobslen.dat};

\nextgroupplot[xlabel={\labelsize $M$}, ylabel={\labelsize $\sigma$}, scaled x ticks=base 10:-3, every x tick scale label/.style={at={(xticklabel* cs:1.15,0.25cm)}, anchor=near xticklabel}] 
\addplot3[surf, shader=flat] file {./data/phasediag/cesnnpclinf-obslensigman.dat};
\nextgroupplot[colorbar, xlabel={\labelsize $M$}, scaled x ticks=base 10:-3, every x tick scale label/.style={at={(xticklabel* cs:1.15,0.25cm)}, anchor=near xticklabel}] 
\addplot3[surf, shader=flat] file {./data/phasediag/ceskmlinf-obslensigman.dat};

\nextgroupplot[xlabel={\labelsize $M$}, ylabel={\labelsize $1/\perr$}, scaled x ticks=base 10:-3, every x tick scale label/.style={at={(xticklabel* cs:1.15,0.25cm)}, anchor=near xticklabel}] 
\addplot3[surf, shader=flat] file {./data/phasediag/cesnnpclinf-obslenp.dat};
\nextgroupplot[colorbar, xlabel={\labelsize $M$}, scaled x ticks=base 10:-3, every x tick scale label/.style={at={(xticklabel* cs:1.15,0.25cm)}, anchor=near xticklabel}] 
\addplot3[surf, shader=flat] file {./data/phasediag/ceskmlinf-obslenp.dat};

\end{groupplot}
\end{tikzpicture}
\caption{\label{fig:phasediag-linf} CE as a function of $d(\powpar{X}{1},\powpar{X}{2})$, $M$, $\sigma$, and $\perr$ for variants of NNPC and KM based on $d_{L^\infty}$, using the same values as in the setup in Fig.~\ref{fig:phasediag} for the model parameters that are not varied.}
\vspace{-0.39cm}
\end{figure}

\subsection{Real data}

We perform experiments on two data sets, namely on human motion data and on EEG data.

\paragraph*{Human motion data} We consider the problem of clustering sequences of human motion data according to the underlying activities performed. Specifically, we consider the experiment conducted in \cite{li2011time, khaleghi2012online}, which 
uses the Carnegie Mellon Motion Capture database\footnote{available at \url{http://mocap.cs.cmu.edu}} containing motion sequences of 149 subjects performing various activities. The clustering algorithm in \cite{li2011time} first fits a linear dynamical system model to each motion sequence and then performs standard $k$-means clustering with the estimated model parameters (organized into vectors) as data points. In \cite{khaleghi2012online} an online clustering algorithm based on KM in combination with distributional distance is proposed. The motion vector-sequences in the Carnegie Mellon Motion Capture database describe the temporal evolution of marker positions on different body parts, recorded through optical tracking. The experiment in \cite{li2011time, khaleghi2012online} is based on subjects \#16 and \#35 for which the database contains 49 and 33 sequences, respectively, labeled either as ``walking'' or ``running''. We cluster the (scalar-valued) sequences describing the motion of the marker placed on the right foot of the subjects. It is argued in \cite{khaleghi2012online} that these sequences can be considered stationary ergodic. 
We assume the number of generative models $L = 2$ to be known and set $q=5$ (good performance was observed for $4 \leq q \leq 10$). For the BT estimator, we use the Bartlett window $g^B_W$, defined in \eqref{eq:defbartlettwin}, with $W$ given by the sequence length, and we normalize the BT PSD estimates to unit power.  
Table \ref{tab:CEMOCAP} lists the CE, the running times in seconds, and for comparison with the results in \cite{li2011time, khaleghi2012online} also the entropy $S$ of the clustering confusion matrix (see \cite[Sec. 6]{li2011time} for the definition of $S$). 
This comparison reveals that for subject \#35 NNPC, KM, and KMit all outperform the algorithm in \cite{li2011time} and match the performance of that in \cite{khaleghi2012online}, while for subject  \#16 NNPC significantly outperforms both the algorithms in \cite{li2011time, khaleghi2012online} as well as KM and KMit.

\begin{table}[t]
\caption{\label{tab:CEMOCAP} CE, $S$, and running time $t$ (in seconds) for clustering of human motion sequences}
\renewcommand{\arraystretch}{1.5}
\setlength\tabcolsep{2.5pt}
\centering
\footnotesize{
\begin{tabular}{| c | c | c | c | c | c | c | c | c | c || c | c |}
\hline
& \multicolumn{3}{c |}{NNPC} & \multicolumn{3}{c|}{KM} & \multicolumn{3}{c||}{KMit} & \cite{khaleghi2012online} & \cite{li2011time}\\
  subject & CE & $S$ & $t$ & CE & $S$ & $t$ & CE & $S$ & $t$ & $S$ & $S$ \\
\hline
\#16 & 0.02 & 0.09 & 0.206 & 0.24 & 0.55 & 0.029 & 0.20 & 0.49 & 0.038 & 0.21 & 0.37\\
\#35 & 0 & 0 & 0.185 & 0 & 0 & 0.017 & 0 & 0 & 0.024 & 0 & 0.10 \\
\hline
\end{tabular}}
\end{table}

\vspace{0.25cm}
\paragraph*{EEG data} 
We perform an experiment similar to that in \cite[Sec. 5]{maharaj2011fuzzy}, which considers clustering of segments of EEG recordings of healthy subjects and of subjects experiencing epileptic seizure according to whether seizure activity is present or not. It is argued in \cite{sanei2008eeg} that EEG recordings can be modeled as stationary ergodic random processes. 
We use subsets A and E of the publicly available\footnote{\url{http://ntsa.upf.edu/downloads/andrzejak-rg-et-al-2001-indications-nonlinear-deterministic-and-finite-dimensional}} EEG data set described in \cite{andrzejak2001indications}. Each of these two subsets contains $100$ EEG segments of $23.6$s duration, acquired at a sampling rate of $173.61$Hz. 
We refer to \cite{andrzejak2001indications} for a more detailed description of acquisition and preprocessing aspects. 

We compare the performance of NNPC, KM, and KMit as a function of $W$ and $q$ (for NNPC). We center each EEG segment by subtracting its (estimated) mean and use a Bartlett window $g^B_W$, as defined in \eqref{eq:defbartlettwin}, of variable length $W$ for the BT PSD estimator. Furthermore, we normalize the PSD estimates to unit power and assume the number of clusters $L=2$ to be known. Fig.~\ref{fig:eegparam} shows the CE obtained for NNPC as a function of the window length $W$ and of $q$, as well as the CE obtained for KM and KMit as a function of $W$. It can be seen that NNPC is robust to small variations of $q$ and $W$ around 
the pair $(q,W)$ corresponding to the minimum CE (marked by a white dot in Fig.~\ref{fig:eegparam}). Similarly, KM and KMit yield a CE close to their respective minima for a large range of values for $W$. In Table \ref{tab:ceeeg}, we report the minimum CE achieved by each algorithm, along with the corresponding running times and CE-minimizing values for $W$ and $q$ (in the case of NNPC), all chosen based on results depicted in Fig. \ref{fig:eegparam}.  
The minimum CE obtained for NNPC is significantly lower than that corresponding to KM and KMit. 

\begin{table}[h]
\caption{\label{tab:ceeeg} Clustering EEG segments: Minimum CE, running time $t$ (in seconds), and corresponding parameter choices.}
\renewcommand{\arraystretch}{1.5}
\centering
\footnotesize{
\begin{tabular}{|c | c | c | c | c |}
\hline
& min CE & $t$ & $W$ & $q$ \\
\hline
NNPC & 0.005 & 0.694 & 840 & 3 \\
KM & 0.360 & 0.482 & 640 & - \\
KMit & 0.095 & 0.954 & 520 & - \\
\hline
\end{tabular}
}
\end{table}

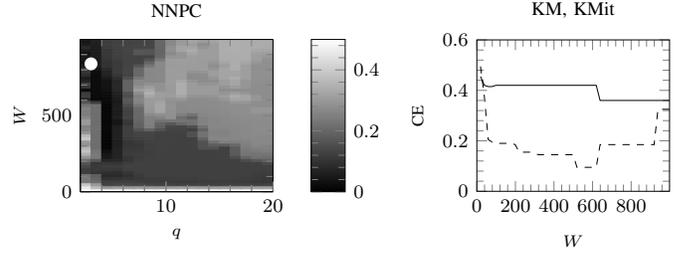
\begin{figure}[h]
\centering
\begin{tikzpicture}[scale = 0.9]
\begin{groupplot}[group style={group size=2 by 1,horizontal sep=0.34\columnwidth,vertical sep=1.3cm,xlabels at=edge bottom, ylabels at=edge left}, 
width=0.5\columnwidth, /tikz/font=\footnotesize, colormap/blackwhite, view={0}{90}, point meta min=0.0, point meta max=0.5, minor tick num=4, xmin=2, xmax=20, ymin=0, ymax=999]
\nextgroupplot[title = NNPC, xlabel={\labelsize $q$}, ylabel={\labelsize $W$},
mesh/ordering=y varies, colorbar]
\addplot3[surf, shader=flat] file {./data/eeg/cesnnpc.dat};
\addplot3+[only marks,mark=*,mark options={white},mark size=2.5] coordinates {(3,839,0.01)};
\nextgroupplot[title = {KM, KMit}, xlabel={\labelsize $W$}, ylabel={\labelsize CE}, y label style={at={(axis description cs:0.2,0.5)},anchor=south},
ymin = 0, ymax = 0.6, xmin=0, xmax=999]
\addplot +[mark=none,solid,black] table[x index=0,y index=1]{./data/eeg/ceskm.dat};
\addplot +[mark=none,dashed,black] table[x index=0,y index=2]{./data/eeg/ceskm.dat};
\end{groupplot}
\end{tikzpicture}
\caption{\label{fig:eegparam} Left: CE of NNPC for EEG recordings as a function of $q$ and $W$. The white dot in the left figure shows the location of minimum CE. Right: CE of KM (solid line) and KMit (dashed line) as a function of $W$.}
\end{figure}

\appendices

\section{Proofs of Theorems \ref{thm:NNPC} and \ref{thm:KM}}
\label{sec:ProofMain}
The central element in the proofs of Theorems \ref{thm:NNPC} and \ref{thm:KM} is the following result, proven in Appendix \ref{sec:PfDistCond}.
\begin{thm} \label{lem:DistCond}
Consider a data set $\cX$ generated according to the statistical data model described in Sec. \ref{sec:mainres}. Then, the clustering condition \eqref{eq:ClusCond} implies that
\begin{equation} \label{eq:DistRelation}
\min_{\substack{k, \l \in [L] \colon \\ k \neq \l}} \, \min_{\substack{i \in [n_\l], \\ j \in [n_k]}} d(\xkh_j,\xlh_i)
> \max_{\l \in [L]} \; \max_{\substack{ i,j \in [n_\l] \colon \\ i \neq j}} d(\xlh_i,\xlh_j)
\end{equation}
holds with probability at least $1 - 6N/M^2$.
\end{thm}
Theorem \ref{lem:DistCond} says that under the clustering condition \eqref{eq:ClusCond} observations stemming from the same generative model are closer (in terms of the distance measure $d$) than observations originating from different generative models. This property 
is known in the clustering literature as the \emph{strict separation property} \cite{balcan2008discriminative}. We now show how Theorems \ref{thm:NNPC} and \ref{thm:KM} follow directly from the strict separation property. 
\paragraph*{Proof of Theorem \ref{thm:NNPC}} \label{sec:ThmNNPCproof}
Under the condition $q \leq \min_{\l \in [L]} (n_\l - 1)$ the NFC property is a direct consequence of \eqref{eq:DistRelation}, which by Theorem \ref{lem:DistCond}, is implied by the clustering condition \eqref{eq:ClusCond}. The condition $q \leq \min_{\l \in [L]} (n_\l - 1)$ is necessary for the NFC property to hold as choosing $q > \min_{\l \in [L]} (n_\l - 1)$ would force NNPC to select observations from $\cX \backslash \cX_\l$ for at least one of the data points $\xlh_i$, thereby resulting in a violation of the NFC property.

\paragraph*{Proof of Theorem \ref{thm:KM}}
\label{sec:ThmKMproof}
The proof is effected by first showing that in Step 3 KM selects an observation with a different underlying generative model in every iteration, i.e., the set of cluster centers $\{ \xh_{c_\l} \}_{\l = 1}^L$ contains exactly one observation from each generative model, provided that the clustering condition \eqref{eq:ClusCond} and hence, by Theorem \ref{lem:DistCond}, \eqref{eq:DistRelation} holds. The argument is then concluded by noting that \eqref{eq:DistRelation} implies directly that the partition $\hat \cX_1,\dots,\hat \cX_L$ obtained in Step 4 corresponds to the true partition $\cX_1,\dots,\cX_L$.

It remains to establish that the cluster centers $\xh_{c_\l}$ selected in Step 3 of KM, indeed, all originate from different generative models. This is accomplished by induction. 
For $v=1$  
the claim holds trivially, as we have selected a single cluster center only, namely $\xh_{c_1}$. 
The base case is hence established. For the inductive step, suppose that after the $v$-th iteration in Step 3 of KM the observations $\{ \xh_{c_1}, \ldots, \xh_{c_v} \}$ all come from different generative models, and assume w.l.o.g. that the generative model underlying $\xh_{c_\l}$ has index $\l$, $\l \in [v]$. 
In iteration $v+1$ (i.e., for the selection of $x_{c_{v+1}}$), we have 
\begin{align}
&\max_{i \in [N]} \min_{\l \in [v]} d(\xh_i, \xlh_{c_\l}) \nonumber \\
& =  \max  \Bigg\{ \! \max_{\substack{k \in [v], \\ i \in [n_k]}} \, \min_{\l \in [v]} d(\xkh_i\!,\xlh_{c_\l}), 
\max_{\substack{k \in [L] \backslash[v], \\ i \in [n_k]}} \, \min_{\l \in [v]} d(\xkh_i\!,\xlh_{c_\l}) \! \Bigg\} \nonumber \\ 
& = \max \Bigg\{  \!\!\!     \underbrace{\max_{\substack{k \in [v],\\ i \in [n_k]}} d(\xkh_i\!,\xkh_{c_k})}_{\leq \underset{\l \in [L]}{\max} \underset{\scriptsize\substack{i,j \in [n_\l] \colon \\ i \neq j}}{\max} d(\xlh_i\!,\xlh_j)} , 
\underbrace{ \max_{\substack{k \in [L] \backslash[v], \\ i \in [n_k]}} \; \min_{ \l \in [v]} d(\xkh_i\!,\xlh_{c_\l})}_{ \substack{
\geq \underset{\scriptsize\substack{k \in [L] \backslash[v], \\ i \in [n_k]}}{\min} \; \underset{\scriptsize \l \in [v]}{\min} d(\xkh_i\!,\xlh_{c_\l}) \\
\geq \underset{\scriptsize\substack{k, \l \in [L] \colon \\ k \neq \l}}{\min} \; \underset{\scriptsize\substack{i \in [n_k], \\ j \in [n_\l]}}{\min} d(\xkh_i\!,\xlh_j)}
} \Bigg\} \label{eq:MinMax1} \\
&\, = \max_{\substack{k \in [L] \backslash[v],\\ i \in [n_k]}} \, \min_{\l \in [v]} d(\xkh_i\!,\xlh_{c_\l}), \label{eq:MinMaxLast}
\end{align}
where we applied \eqref{eq:DistRelation} to get \eqref{eq:MinMaxLast}
 from \eqref{eq:MinMax1}. Note that in the maximization in \eqref{eq:MinMaxLast} $k$ runs over $[L] \backslash [v]$ (i.e., the maximization in \eqref{eq:MinMaxLast} is over the observations in $\cX \backslash (\cX_1 \cup \dots \cup \cX_v)$), which implies that $\xh_{c_{v+1}}$ is guaranteed to correspond to a generative model that is different from those underlying $\xh_{c_1}, \dots, \xh_{c_v}$. This completes the induction argument.

\section{Proof of Theorem \ref{lem:DistCond}} \label{sec:PfDistCond}

We start by quantifying the deviation of the estimated distances $d(\xkh_j, \xlh_i)$ from the true distances $d(\Xk,\Xl)$ due to the PSD estimation error caused by finite observation length, noise, and missing entries.

Let $\slt(f) \defeq \sl(f) + \sigma^2$, $f \in [0,1)$, $\l \in [L]$, be the PSD of the noisy observation $\Xlt$ and denote the corresponding ACF by $\rlt$. 
With $\slh_i (f)$ as defined in \eqref{eq:BlackmanTukey} (recall that we use the modified window $\gh[m] = g[m]/ \errwin[m]$ in the BT estimator \eqref{eq:BlackmanTukey}), set $\el_i(f) \defeq \slh_i (f) -  \slt(f)$, 
and let $\varepsilon \defeq \max_{\l \in [L], i \in [n_\l]} \sup_{f \in [0,1)} \abss{\el_i(f)}$. We have for all $k, \l \in [L]$, $j \in [n_k]$, $i \in [n_\l]$, 
\begin{align}
&d(\xkh_j, \xlh_i) = \frac{1}{2} \int_0^1 \abs{\skh_j(f) - \slh_i(f)} \d f \nonumber \\
&= \frac{1}{2} \! \int_0^1\! \left\vert  \sk(f) \!+\! \sigma^2 \!+\! \ek_j(f) - ( \sl(f) \!+\! \sigma^2 \!+\! \el_i(f) ) \right\vert \d f \nonumber \\ 
&\leq \frac{1}{2} \! \int_0^1 \!  \left\vert \sk(f) \!-\! \sl(f) \right\vert \d f + \frac{1}{2} \! \int_0^1 \! \left\vert \ek_j(f) - \el_i(f) ) \right\vert \d f \nonumber \\
&\leq  d(\Xk,\Xl) + \frac{1}{2} \int_0^1 \abs{\ek_j(f)}  \d f + \frac{1}{2} \int_0^1 \abs{\el_i(f)} \d f  \label{eq:UBTriang1}\\
&\leq  d(\Xk,\Xl) + \varepsilon. \label{eq:DistUB}
\end{align}
Applying the reverse triangle inequality, it follows similarly that
\begin{equation}
d(\xkh_j, \xlh_i) \geq d(\Xk,\Xl) - \varepsilon, \label{eq:DistLB}
\end{equation}
for all $k, \l \in [L]$, $j \in [n_k]$, $i \in [n_\l]$. 
Replacing the RHS of \eqref{eq:DistRelation} by the upper bound in \eqref{eq:DistUB} and the LHS by the lower bound in \eqref{eq:DistLB}, we find that \eqref{eq:DistRelation} is implied by
\begin{equation} \label{eq:DistCondLBeps}
\min_{k, \l \in [L] \colon k \neq \l} d(\Xk, \Xl) > 2 \varepsilon.
\end{equation}

We continue by upper-bounding $\varepsilon$. To this end, define $\mQ_m \in \{0,1\}^{M \times M}$ according to $(\mQ_m)_{u,v} \allowbreak= 1$, if $v - u = m$, and $(\mQ_m)_{u,v} = 0$, else, and let $\mGh(f) \defeq \sum_{m \in \mc M} \gh[m] \cos(2 \pi f m) \mQ_m$. 
Now, with $\vx \in \reals^M$ the random vector whose elements are given by $\xlh_i$, it holds for $m \in \mc M = \{ -M +1, -M + 2, \dots, M - 1\}$ that
\begin{align}
\rlh_i [m] = \frac{\transp{\vx} \mQ_m \vx}{M} = \frac{\transp{\vy} \transp{\mC} \mP^\top_{\bxi} \mQ_m \mP_{\bxi} \mC \vy}{M}, \nonumber 
\end{align}
where we used $\vx = \mP_{\bxi} \mC \vy$, with the entries of $\vy$ i.i.d. standard normal, $\mC = (\mR + \sigma^2 \mI)^{1/2} \in \reals^{M \times M}$ with $\mR_{v,w} = \rl[w-v]$ the (Toeplitz) covariance matrix corresponding to $M$ consecutive elements of $\Xlt$, and $\bxi \in \{0,1\}^M$ indicates the locations of the observed entries of $\vx$. Note that $\mR$ is identical for all contiguous length-$M$ segments of $\Xlt$ thanks to stationarity, and $\mC$ is symmetric because $\mR + \sigma^2 \mI$ is symmetric.  
We next develop an upper bound on $\varepsilon$ according to
\begin{align}
&\sup_{f \in [0,1)} \abs{\el_i(f)} \nonumber = \sup_{f \in [0,1)} \abs{\slh_i(f) - \slt(f)} \nonumber \\
&= \sup_{f \in [0,1)} \bigg\vert \!\sum_{m \in \mc M} \gh[m] \rlh_i [m] e^{-\im 2 \pi f m} - \!  \sum_{m \in \mb Z} \rlt [m] e^{-\im 2 \pi f m} \bigg\vert \nonumber\\
&= \sup_{f \in [0,1)} \bigg\vert \! \!\underbrace{\sum_{m \in \mc M } \frac{\gh[m]}{M}\left( \transp{\vx} \mQ_m \vx - \EX{\vy}{ \transp{\vx} \mQ_m \vx } \right) e^{-\im 2 \pi f m}}
_{\substack{\frac{1}{M} \big( \transp{\vx} \left( \sum_{m \in \mc M} \gh[m] \cos(2 \pi fm) \mQ_m \right) \vx \\ \hspace{2cm} - \mb E_{\vy} \left[ \transp{\vx} \left( \sum_{m \in \mc M} \gh[m] \cos(2 \pi fm) \mQ_m \right) \vx \right] \big)}} \nonumber \\ 
& + \sum_{m \in \mc M} \frac{\gh[m]}{M} \, \EX{\vy}{ \transp{\vx} \mQ_m \vx } e^{-\im 2 \pi f m}  
- \sum_{m \in \mb Z} \rlt [m] e^{-\im 2 \pi f m} \bigg\vert \nonumber \\ 
& \leq  \sup_{f \in [0,1)} \bigg\vert\frac{1}{M} \left( \transp{\vy} \transp{\mC} \transp{\mP}_{\bxi} \mGh(f) \mP_{\bxi} \mC \vy \right.\nonumber \\
& \qquad \qquad \; \; \underbrace{ \qquad \qquad \quad \left. - \,\EX{\vy}{ \transp{\vy} \transp{\mC} \transp{\mP}_{\bxi} \mGh(f)  \mP_{\bxi} \mC \vy} \right) }_{ \eqdef \al_i(f)} \bigg\vert \nonumber \\
&  \qquad \qquad + \bigg\vert \underbrace{\sum_{m \in \mc M} \frac{\gh[m]}{M} \EX{\vy}{ \transp{\vx} \mQ_m \vx } - \sum_{m \in \mb Z} \rlt [m] }_{ \eqdef \bl_i} \bigg\vert, \label{eq:SupELIFinal}
\end{align}
where we used the fact that $\gh[m] (\transp{\vx} \mQ_m \vx - \EX{\vy}{ \transp{\vx} \mQ_m \vx }\!)$ is a real-valued even sequence ($\gh[m]$ and $\transp{\vx} \mQ_m \vx = M \rlh_i[m]$ are real-valued even by definition, the latter property implies that $\EX{\vy}{ \transp{\vx} \mQ_m \vx }$ is also real-valued and even). 
It now follows from \eqref{eq:SupELIFinal} that 
\[ \varepsilon \leq \max_{\l \in [L], i \in [n_\l]} \big( \sup_{f \in [0,1)} \big\vert \al_i (f) \big\vert + \big\vert \bl_i \big\vert \big) \] and hence \eqref{eq:ClusCond} implies \eqref{eq:DistRelation} via \eqref{eq:DistCondLBeps} on the event $\Fstar \defeq \bigcap_{\l \in [L], i \in [n_\l]} (\Fali \cap \Fbli)$ with
\begin{align*}
	\Fali &\defeq  \left\{ \sup_{f \in [0,1)}  \abs{\al_i (f)}  <  \frac{4 A (B+ \sigma^2)}{\perr^2} \sqrt{\frac{2 \log M}{M}} \right\} \; \; \text{and} \nonumber \\
	\Fbli &\defeq  \left\{ \abs{\bl_i}  < 8 (1+ p) \frac{A(1+ \sigma^2)}{\perr^2} \sqrt{\frac{\log M}{M} } + \mu_{\max} \right\}.
\end{align*}
With the upper bound on $\PR{\smash{\Falic}}$ resulting from \eqref{eq:PrFaUB} and that on $\PR{\smash{\Fblic}}$ in \eqref{eq:PrFbUB}, 
application of the union bound according to
\begin{align}
\PR{\Fstar} \geq 1 - \sum_{\l \in [L], i \in [n_\l]} \left( \PR{\Falic} + \PR{\Fblic} \right) \geq 1 - \frac{6N}{M^2}
\end{align}
completes the proof.

We proceed to the upper bound on $\PR{\smash{\Falic}}$.

\paragraph*{Upper bound on $\PR{\smash{\Falic}}$\normalfont} 
Conditioning on $\bxi$ and setting $\mB \defeq \transp{\mC} \transp{\mP}_{\bxi} \mGh(f) \mP_{\bxi} \mC$, we establish an upper bound on the tail probability of $\sup_{f \in [0,1)}  \abs{\al_i (f)}$ by invoking a well-known concentration of measure result for quadratic forms in Gaussian random vectors \cite[Lem. 1]{demanet2012matrix}, namely
\begin{align}
&\mathrm{P} \Big[ \abs{ \transp{\vy} \mB \vy - \E{\transp{\vy} \mB \vy}} \nonumber \\
&\qquad \quad   \geq \norm[F]{\mB + \transp{\mB}}\sqrt{\delta} + 2 \norm[2 \to 2]{\mB} \delta \big\vert \bxi \Big] \leq 2 e^{-\delta}. \label{eq:GaussQuadForm}
\end{align}
Next, we note that $\norm[F]{\smash{\mB + \transp{\mB}}} \leq 2 \norm[F]{\mB} \leq 2 \sqrt{M} \norm[2 \to 2]{\mB}$ and 
\begin{align}
\norm[2 \to 2]{\mB} &\leq \norm[2 \to 2]{\smash{\underbrace{\mC}_{=\, \mR + \sigma^2 \mI}}}^2 \underbrace{\norm[2 \to 2]{\mP_{\bxi}}^2}_{\leq 1} \norm[2 \to 2]{\smash{\mGh(f)}} \nonumber \\ 
&\leq \frac{A (B + \sigma^2)}{\perr^2}, \nonumber
\end{align}
where the second inequality follows as both $\mR$ and $\mGh(f)$ are symmetric Toeplitz matrices and hence, by \cite[Lem. 4.1]{gray2006toeplitz}, $\norm[2 \to 2]{\mR} \leq \sup_{f \in [0,1)} \sl(f) \leq B$ and 
\begin{align}
\norm[2 \to 2]{\smash{\mGh(f)}} &\leq \sup_{f' \in [0,1)} \gh(f') \nonumber \\ 
&=\sup_{f' \in [0,1)} \frac{1}{p^2} g(f') + \underbrace{\left(\frac{1}{p} - \frac{1}{p^2}\right)}_{\leq 0} \underbrace{g[0]}_{= 1} \nonumber \\
&\leq \frac{1}{p^2} \sup_{f' \in [0,1)} g(f') =\frac{A}{p^2}, \label{eq:GhatOpUB}
\end{align}
where  
we used $\gh[m] = (1/p^2) g[m] + (1/p - 1/p^2)g[0]\delta[m]$. 
Now, setting $\delta = 2 \log (M)$ in \eqref{eq:GaussQuadForm} and using $\delta/M \leq \sqrt{\delta/M} < 1$, for $M \geq 1$, yields
\begin{align}
&\PR{\Falic \big\vert \bxi} = \mathrm{P}\left[\sup_{f \in [0,1)} \abs{\al_i (f)} \right.  \nonumber \\
& \qquad \qquad \qquad \qquad \left. \geq \frac{4 A (B+ \sigma^2)}{\perr^2} \sqrt{\frac{2 \log M}{M}} \bigg\vert \bxi\right] \leq \frac{2}{M^2}. \label{eq:PrFaUB}
\end{align}
The proof is concluded by noting that this bound holds uniformly over $\bxi \in \{0,1\}^M$ so that $\PR{\smash{\Falic}} \leq 2 / M^2$.

\paragraph*{Upper bound on $\PR{\smash{\Fblic}}$\normalfont}
Setting $\mGt \defeq \sum_{m \in \mc M} \gh [m] \mQ_m$, we start by rewriting the first sum in the definition of $\bl_i$ in \eqref{eq:SupELIFinal} as 
\begin{align}
&\sum_{m \in \mc M} \frac{\gh [m]}{M} \, \EX{\vy}{ \transp{\vx} \mQ_m \vx }
= \frac{1}{M} \, \EX{\vy}{ \transp{\vx} \bigg(\sum_{m \in \mc M} \gh [m]  \mQ_m \bigg) \vx } \nonumber \\
& \qquad \qquad = \frac{1}{M} \, \EX{\vy}{ \transp{\vy} \transp{\mC} \transp{\mP}_{\bxi} \mGt  \mP_{\bxi} \mC \vy } \nonumber \\
& \qquad \qquad = \frac{1}{M} \, \tr(\transp{\mC} \transp{\mP}_{\bxi} \mGt  \mP_{\bxi} \mC) \nonumber \\
& \qquad \qquad = \frac{1}{M} \, \tr(\transp{\mP}_{\bxi} \mGt  \mP_{\bxi} \underbrace{\mC \transp{\mC}}_{= \mR + \sigma^2 \mI}) \label{eq:misstrace1} \\
& \qquad \qquad = \frac{1}{M}\sum_{u,v \in [M]} \xi_u \xi_v \mGt_{u,v} (\underbrace{\mR_{v,u}}_{= \mR_{u,v}} + \sigma^2 \delta[u-v]) \label{eq:misstrace2} \\
& \qquad \qquad = \frac{1}{M}\transp{\bxi} (\mGt \circ (\mR + \sigma^2 \mI)) \bxi. \label{eq:misstrace3}
\end{align}
Now, setting $\mD \defeq \mGt \circ (\mR + \sigma^2 \mI)$ and using \eqref{eq:misstrace3}, we have
\begin{align}
&\left| \bl_i \right| \nonumber \\
&= \bigg\vert \frac{1}{M} \transp{\bxi} \mD \bxi - \frac{1}{M} \, \E{\transp{\bxi} \mD \bxi} +  \frac{1}{M} \, \E{\transp{\bxi} \mD \bxi} \!-\! \sum_{m \in \mb Z} \! \rlt [m] \bigg\vert \nonumber \\
&\leq \abs{ \frac{1}{M} \! \left( \transp{\bxi} \mD \bxi - \E{\transp{\bxi} \mD \bxi} \right) } \!+\! \bigg\vert \frac{1}{M} \, \E{\transp{\bxi} \mD \bxi} \!-\! \sum_{m \in \mb Z} \! \rlt [m] \bigg\vert \label{eq:bliUB0} \\
&\leq \bigg\vert \underbrace{\frac{1}{M} \left( \transp{\bxi} \mD \bxi - \E{\transp{\bxi} \mD \bxi} \right)}_{\eqdef \cl_i} \bigg\vert + \mu_{\max}. \label{eq:bliUB}
\end{align}
Here, the last inequality is a consequence of the following upper bound on the second term in~\eqref{eq:bliUB0}
\begin{align}
&\bigg\vert \frac{1}{M} \E{\transp{\bxi} \mD \bxi} - \sum_{m \in \mb Z} \rlt [m] \bigg\vert \nonumber \\
&= \bigg\vert \frac{1}{M}  \sum_{v,w \in [M]} \E{\xi_v \xi_w} \mGt_{v,w} (\mR_{v,w} + \sigma^2 \delta[v-w]) \nonumber \\
& \qquad \qquad - \sum_{m \in \mb Z} \rlt [m] \bigg\vert \label{eq:mubnd1} \\
&=  \bigg\vert \frac{1}{M} \sum_{v,w \in [M]} \underbrace{\E{\xi_v \xi_w} \gh[v-w]}_{u[v-w] \gh[v-w]= g[v-w]} \rlt[v-w] - \sum_{m \in \mb Z}\rlt [m] \bigg\vert \nonumber \\
&=  \bigg\vert \underbrace{g[0]}_{=1} (\rl[0] + \sigma^2) - (\rl[0] + \sigma^2) \nonumber \\
& \qquad \qquad + \sum_{m \in \mc M \backslash \{0\}} \left(\frac{M - \abs{m}}{M} g[m] \rl[m] - \rl[m]\right) \nonumber \\
& \qquad \qquad - \sum_{m \in \mb Z \backslash \mc M} \rl[m] \bigg\vert \label{eq:mubnd2} \\
&\leq \sum_{m \in \mb Z} \vert h [m] \vert \vert \rl[m] \vert \nonumber \\
&\leq \mu_{\max},
\end{align}
where \eqref{eq:mubnd1} follows from the equality \eqref{eq:misstrace3}$=$\eqref{eq:misstrace2} and from $\rlt[m] = \rl[m] + \sigma^2 \delta[m]$. We continue by establishing a bound on the tail probability of $|\cl_i|$. To this end, we note that 
\begin{align}
\norm[2\to2]{\mD} &=  \norm[2 \to 2]{\smash{\mGt \circ (\mR + \sigma^2 \mI)}} \nonumber \\
&= (1+\sigma^2)\norm[2 \to 2]{\mGt \circ \left(\frac{\mR + \sigma^2 \mI}{1+\sigma^2} \right) } \nonumber \\
&\leq (1+\sigma^2)\norm[2 \to 2]{\smash{\mGt}} \nonumber \\
&\leq \frac{A(1 + \sigma^2)}{\perr^2}, \label{eq:DUB}
\end{align}
where we used the fact that $(\mR + \sigma^2 \mI)/(1+\sigma^2)$ is a symmetric positive semi-definite matrix with ones on its main diagonal,  
and we employed \cite[Thm. 5.5.11]{horn1991topics} in the first inequality, and steps analogous to those in \eqref{eq:GhatOpUB} to obtain the second inequality. 

Now, using \eqref{eq:bliUB} we get
\begin{align}
\PR{\Fblic}
&\leq \PR{\abs{\cl_i} \geq 8 (1+ p) \frac{A(1+ \sigma^2)}{\perr^2} \sqrt{\frac{\log M}{M}}} \nonumber \\
&< \PR{\abs{\cl_i} > 8 (1+ p) \norm[2 \to 2]{\mD} \sqrt{\frac{\log M}{M}}} < \frac{4}{M^2}, \label{eq:PrFbUB}
\end{align}
where the second inequality follows from the upper bound on $\norm[2 \to 2]{\mD}$ in \eqref{eq:DUB} and the third inequality is an application of Lemma \ref{le:boolquadform} with $\mH \defeq \mD$ and $t \defeq 8(1+p) \norm[2 \to 2]{\mD} \sqrt{M \log M}$.

A final remark concerns the concentration inequality for quadratic forms in Boolean random vectors reported in the following Lemma \ref{le:boolquadform}. Such concentration inequalities, or more generally, concentration inequalities for multivariate polynomials of Boolean random variables have been studied extensively in the context of random graph theory \cite{schudy2012concentration}.  
Unfortunately, the bounds available in the literature typically come in terms of functions of the entries of $\mH$ that do not lead to crisp statements in the context of the process clustering problem considered here. We therefore develop a new concentration result in Lemma \ref{le:boolquadform}, which depends on $\norm[2\to2]{\mH}$ only. The proof of this result is based on techniques developed in \cite{hansonwright2013rudelson}.

\begin{lem} \label{le:boolquadform}
Let $\mH \in \reals^{M \times M}$ be a (deterministic) symmetric matrix and let $\bxi \in \{0,1\}^M$ be a random vector with i.i.d. Bernoulli entries drawn according to $\PR{\xi_i = 1} =  1- \PR{\xi_i = 0} = p$, $i \in [M]$. Then, we have
\begin{align}
&\PR{\abs{\transp{\bxi} \mH \bxi -  \E{\transp{\bxi} \mH \bxi}} > t} \nonumber \\
& \qquad \qquad \qquad < 4 \exp\left( - \frac{t^2}{32 (1+p)^2 M \norm[2 \to 2]{\mH}^2} \right). \label{eq:boolquadform}
\end{align}
\end{lem}

\begin{proof}
The proof is effected by adapting the proof of \cite[Thm. 1.1]{hansonwright2013rudelson}, which provides a concentration inequality for quadratic forms in zero-mean subgaussian random vectors. 
We start by decomposing $\transp{\bxi} \mH \bxi - \E{\transp{\bxi} \mH \bxi}$ according to
\begin{align}
&\transp{\bxi} \mH \bxi - \E{\transp{\bxi} \mH \bxi} \nonumber \\
& \quad= \sum_{i \in [M]} \mH_{i,i} (\xi^2_i - \E{\xi^2_i}) + \sum_{\substack{i,j \in [M] \colon \\ i \neq j}} \mH_{i,j} (\xi_i \xi_j - \E{\xi_i \xi_j}) \nonumber \\
& \quad= \underbrace{\sum_{i \in [M]} \mH_{i,i} (\xi_i - p)}_{\eqdef S_\mathrm{diag}} + \underbrace{\sum_{\substack{i,j \in [M] \colon \\ i \neq j}} \mH_{i,j} (\xi_i \xi_j - p^2)}_{ \eqdef \So}, \nonumber
\end{align}
where we used the fact that the $\xi_i$, $i \in [M]$, are $\{0,1\}$-valued and statistically independent. 
Now, we have 
\begin{align}
&\PR{ \abs{\transp{\bxi} \mH \bxi - \E{\transp{\bxi} \mH \bxi}} > t} \leq \PR{ \abs{S_\mathrm{diag}} + \abs{\So} > t} \nonumber \\
& \qquad \leq \PR{ \abs{S_\mathrm{diag}}  > t/2} + \PR{ \abs{\So} > t/2} \label{eq:splitprobbnd0} \\
& \qquad \leq 2\exp \left(-\frac{2 t^2}{\sum_{i \in [M]}(\mH_{i,i})^2} \right) \nonumber \\
& \qquad  \qquad \qquad + 2 \exp\left( - \frac{t^2}{32 (1+p)^2 M \norm[2 \to 2]{\mH}^2} \right) \label{eq:splitprobbnd} \\
& \qquad < 4 \exp\left( - \frac{t^2}{32 (1+p)^2 M \norm[2 \to 2]{\mH}^2} \right), \nonumber
\end{align}
where 
\eqref{eq:splitprobbnd} follows from the upper bounds on $\PR{ \abs{S_\mathrm{diag}}  > t/2}$ and $\PR{ \abs{\So} > t/2}$ established below, 
and the last inequality is thanks to $\sum_{i \in [M]} (\mH_{i,i})^2 \leq M \max_{i \in [M]} (\mH_{i,i})^2 \leq M \norm[2 \to 2]{\mH}^2$ obtained from $\norm[2 \to 2]{\mH}^2 \!= \!\max_{\norm[2]{\vx} = 1} \norm[2]{\mH \vx}^2 \geq \! \max_{\norm[2]{\vx} = 1, \vx \in \{0,1\}^M} \!\norm[2]{\mH \vx}^2 = \! \max_{i \in [M]} \sum_{j \in[M]} (\mH_{j,i})^2 \allowbreak\geq \max_{i \in [M]} (\mH_{i,i})^2$. 

\paragraph*{Upper bound on $\PR{\abs{S_\mathrm{diag}} > t/2}$\normalfont}
Note that the $\mH_{i,i} (\xi_i - p)$, $i \in [M]$, are independent, bounded, zero-mean random variables with $a_i \leq \mH_{i,i} (\xi_i - p) \leq b_i$, $a_i, b_i \in \reals$,
$i \in [M]$. We can therefore apply Hoeffding's inequality \cite[Thm. 2.8]{boucheron2013concentration}, which upon noting that $(b_i - a_i)^2 = \mH_{i,i}^2$ yields
\begin{equation}
\PR{\abs{S_\mathrm{diag}} > t/2} < 2 \exp \left(-\frac{2 t^2}{\sum_{i \in [M]}(\mH_{i,i})^2} \right). \nonumber 
\end{equation}

\paragraph*{Upper bound on $\PR{\abs{\So} > t/2}$\normalfont}
We start by decoupling \cite[Sec. 8.4]{foucart_mathematical_2013} the sum $\So$ over the off-diagonal entries of $\mH$, then upper-bound the moment generating function of $\So$, and use the resulting upper bound to get an upper bound on $\PR{\So > t/2}$ via the exponential Chebyshev inequality. The final result follows by noting that $\PR{\So > t/2} = \PR{\So < -t/2}$ and applying the union bound. 

To decouple $\So$, consider i.i.d. Bernoulli random variables $\nu_i \in \{0,1\}$, $i \in [M]$, with $\PR{\nu_i = 0} = \PR{\nu_i = 1} = 1/2$, and set $\bnu = \transp{[\nu_1 \; \ldots \; \nu_M]}$. With
\begin{equation}
\Sonu \defeq \sum_{i,j \in [M]} \nu_i(1-\nu_j)\mH_{i,j} (\xi_i -p) (\xi_j + p), \nonumber 
\end{equation}
we have $\So = 4 \EX{\bnu}{\Sonu}$ thanks to the symmetry of $\mH$ (i.e., $\mH_{i,j} = \mH_{j,i}$), and $\E{\nu_i(1-\nu_j)} = 1/4$, for $i \neq j$, and $\E{\nu_i(1-\nu_j)} = 0$, for $i = j$. Setting $\mc I_\nu \defeq \{ i \in [M] \colon \nu_i = 1 \}$, we can express $\Sonu$ as
\begin{align}
\Sonu &= \sum_{i \in \mc I_\nu, j \in \comp{\mc I_\nu}} \mH_{i,j} (\xi_i -p) (\xi_j + p) \nonumber \\
&= \sum_{i \in \mc I_\nu} (\xi_i -p) \bigg( \sum_{j \in \comp{\mc I_\nu}} \mH_{i,j} (\xi_j + p) \bigg). \label{eq:Sonu}
\end{align}
We continue by upper-bounding the moment generating function of $\So$ via Jensen's inequality according to
\begin{align}
\EX{\bxi}{\exp(\lambda \So)} &= \EX{\bxi}{\exp(\lambda 4 \EX{\bnu}{\Sonu})} \nonumber \\
&\leq \EX{\bxi,\bnu}{\exp(4 \lambda \Sonu)}\!, \label{eq:expSonuJensen}
\end{align}
where $\lambda > 0$ is a deterministic parameter. It follows from \eqref{eq:Sonu} that $\Sonu$, conditioned on $\bnu$ and on the $\xi_j$, with $j \in \comp{\mc I_\nu}$, is a linear combination of independent bounded zero-mean random variables. 
We therefore have 
\begin{align}
&\EX{\xi_i, i \in \mc I_\nu}{\exp(4 \lambda \Sonu)} \nonumber \\
&= \EX{\xi_i, i \in \mc I_\nu}{\exp\left(4 \lambda \sum_{i \in \mc I_\nu} (\xi_i -p) \bigg( \sum_{j \in \comp{\mc I_\nu}} \mH_{i,j} (\xi_j + p) \bigg)\right)} \nonumber \\
&= \prod_{i \in \mc I_\nu} \EX{\xi_i}{\exp\left(4 \lambda (\xi_i -p) \bigg( \sum_{j \in \comp{\mc I_\nu}} \mH_{i,j} (\xi_j + p) \bigg)\right)} \label{eq:mgfcondUB1} \\
&\leq \prod_{i \in \mc I_\nu} \exp \left(2 \lambda^2 \bigg( \sum_{j \in \comp{\mc I_\nu}} \mH_{i,j} (\xi_j + p) \bigg)^2 \right) \label{eq:mgfcondUB2} \\
&=  \exp \Bigg(2 \lambda^2 \sum_{i \in \mc I_\nu} \bigg( \underbrace{\sum_{j \in \comp{\mc I_\nu}} \mH_{i,j} (\xi_j + p)}_{= \mH_i (\mI - \mP_{\bnu}) (\bxi + p {\mathbf 1})} \bigg)^2 \Bigg) \nonumber \\
&= \exp\left( 2\lambda^2 \norm[2]{\mP_{\bnu} \mH (\mI - \mP_{\bnu}) (\bxi + p {\mathbf 1})}^2 \right) \nonumber \\
&\leq \exp\bigg( 2\lambda^2 
\underbrace{\norm[2 \to 2]{\mP_{\bnu}}^2}_{\leq 1}
\norm[2 \to 2]{\mH}^2
\underbrace{\norm[2 \to 2]{\mI - \mP_{\bnu}}^2}_{\leq 1}
\underbrace{\norm[2]{\bxi + p {\mathbf 1}}^2}_{\leq M (1+p)^2}
 \bigg) \nonumber \\ 
&\leq \exp\left( 2\lambda^2 (1+p)^2 M \norm[2 \to 2]{\mH}^2 \right)\!, \label{eq:mgfcondUBfin}
\end{align}
where we used the independence of the $\xi_i$, $i \in \mc I_\nu$, to get \eqref{eq:mgfcondUB1}, and Hoeffding's Lemma 
in the step leading from \eqref{eq:mgfcondUB1} to \eqref{eq:mgfcondUB2}. Note that instead of Hoeffding's Lemma we could also apply \cite[Thm. 2.1]{buldygin2013sub} to get a sharper bound on \eqref{eq:mgfcondUB1}, 
but this would not lead to a different scaling behavior of \eqref{eq:boolquadform} in terms of $\perr$ or $M$. 

Combining \eqref{eq:mgfcondUBfin} with \eqref{eq:expSonuJensen} and noting that the bound \eqref{eq:mgfcondUBfin} does not depend on $\bnu$ and $\xi_j$, $j \in \comp{\mc I_\nu}$, it follows that
\begin{align}
&\EX{\bxi,\bnu}{\exp(\lambda \So)}
\leq \EX{\bxi,\bnu}{\exp(4 \lambda \Sonu)} \nonumber \\
&\qquad= \EX{\bnu}{\EX{\xi_j, j \in \comp{\mc I_\nu}}{\EX{\xi_i, i \in \mc I_\nu}{\exp(4 \lambda \Sonu)}}} \nonumber \\
&\qquad\leq \EX{\bnu}{\EX{\xi_j, j \in \comp{\mc I_\nu}}{\exp\left( 2\lambda^2 (1+p)^2 M \norm[2 \to 2]{\mH}^2 \right)}} \nonumber \\
&\qquad= \exp\left( 2\lambda^2 (1+p)^2 M \norm[2 \to 2]{\mH}^2 \right).\label{eq:momgenbnd}
\end{align}
We finally use \eqref{eq:momgenbnd} and the exponential Chebyshev inequality to get the upper bound
\begin{align}
\PR{\So > t/2} 
&\leq  \exp\left( -\lambda t / 2+ 2\lambda^2 (1+p)^2 M \norm[2 \to 2]{\mH}^2 \right), \label{eq:chebbnd}
\end{align} 
which holds for all $\lambda >0$. Minimizing \eqref{eq:chebbnd} over $\lambda > 0$ yields
\begin{equation}
\PR{\So > t/2} \leq  \exp\left( - \frac{t^2}{32 (1+p)^2 M \norm[2 \to 2]{\mH}^2} \right). \label{eq:proboffdiagbound}
\end{equation}
\end{proof}

\vspace{-1cm}
\section{Proof of Proposition \ref{pr:PropInnerProd}}
\label{sec:TSCpropproof}
Recall that $\vxl_i = \Cl \vyl_i$, $\l \in [L]$, $i \in [n_\l]$, where $\vyl_i$ is an i.i.d. standard normal random vector and $\Cl \defeq (\Rlt) \vphantom{\Rlt}^{1/2}$. Setting $\sigkl \defeq \norm[2]{\smash{\transp{\Ck\!} \Cl \vyl_i}}$, 
conditional on 
\vspace{0.075cm}
$\vyl_i$, $\innerprod{\smash{\vxk_j}}{\smash{\vxl_i}}$ and $\innerprod{\smash{\vxl_v}}{\smash{\vxl_i}}$, for $k \neq \l$ and $v \neq i$, are independent (as a consequence of the mutual independence of the $\vxl_i$, $\l \in [L]$, $i \in [n_\l]$, which is by assumption) and distributed according to $\mc N(0, \sigkl^2)$ and $\mc N(0, \sigll^2)$, respectively. 
\vspace{0.075cm}
Conditional on $\vyl_i$, or equivalently, conditional on 
$\sigkl$ and $\sigll$, $\abs{\innerprod{\smash{\vxk_j}}{\smash{\vxl_i}}}$ and $\abs{\innerprod{\smash{\vxl_v}}{\smash{\vxl_i}}}$ hence have half-normal distributions 
and we get
\begin{align}
&\PR{\abs{\innerprod{\vxk_j}{\vxl_i}} < \abs{\innerprod{\vxl_v}{\vxl_i}} \bigg\vert \frac{\sigkl}{\sigll}} \nonumber \\
& \qquad = \int_0^\infty \!\! \frac{\sqrt{2}}{\sigll \sqrt{\pi}} e^{-\frac{x^2}{2 {\sigll}^2}} \int_0^x  \frac{\sqrt{2}}{\sigkl \sqrt{\pi}} e^{-\frac{y^2}{2 {\sigkl}^2}} \d y \, \d x \nonumber \\
& \qquad = \int_0^\infty \!\! \frac{\sqrt{2}}{\sigll \sqrt{\pi}} e^{-\frac{x^2}{2 {\sigll}^2}} \erf \! \left(\frac{x}{\sigkl \sqrt{2}}\right) \d x  \nonumber \\
& \qquad = 1 - \frac{2}{\pi} \atan \! \left(\frac{\sigkl}{\sigll}\right), \label{eq:innerprodcondprob}
\end{align}
where we used the integral formula \cite[Eqn. 2, p. 7]{ng1969table} $\int_0^\infty \erf(ax) e^{-b^2 x^2} \d x = (\pi/2 - \atan(b/a))/(b \sqrt{\pi})$, with $a = 1/(\sigkl \sqrt{2})$ and $b = 1/(\sigll \sqrt{2}) \allowbreak$ to arrive at \eqref{eq:innerprodcondprob}.

Denoting the probability density function of $\sigkl / \sigll$ by $p_\sigma$, we get for fixed $\beta > 0$,
\begin{align}
&\PR{\abs{\innerprod{\vxk_j}{\vxl_i}} \geq \abs{\innerprod{\vxl_v}{\vxl_i}} } \nonumber \\
&=\int_0^\infty \! \left( 1 \!-\! \PR{\abs{\innerprod{\vxk_j}{\vxl_i}} < \abs{\innerprod{\vxl_v}{\vxl_i}} \bigg\vert x} \right) p_\sigma(x) \d x \nonumber \\
&= \int_0^\infty \frac{2}{\pi} \atan \! \left(x \right) p_\sigma(x) \d x \nonumber \\
&\geq \int_\beta^\infty \frac{2}{\pi} \atan \! \left(x \right) p_\sigma(x) \d x \nonumber \\
&\geq \frac{2}{\pi} \atan \! \left( \beta \right) \int_\beta^\infty p_\sigma(x) \d x \nonumber \\
& = \frac{2}{\pi} \atan \! \left( \beta \right) \PR{\frac{\sigkl}{\sigll} \geq \beta}. \label{eq:CondProbLB}
\end{align}

We continue by setting
\begin{equation}
\beta \defeq \frac{\sqrt{\tr \! \left(\Rkt \Rlt \right)}}{5 \sqrt{3} \sqrt{\tr \! \left(\Rlt \Rlt \right)}} \nonumber
\end{equation}
and obtain
\begin{align}
 \PR{\frac{\sigkl}{\sigll} \geq \beta} 
 &\geq \mathrm{P}\left[ \left\{ \sigkl \geq \frac{1}{\sqrt{3}} \sqrt{\tr \! \left(\Rkt \Rlt \right)} \right\} \right.\nonumber \\
 & \qquad \qquad \quad \left. \cap \left\{ \sigll \leq 5 \sqrt{\tr \! \left(\Rlt \Rlt \right)} \right\} \right] \nonumber \\
 &\geq 1 - \PR{ \sigkl < \frac{1}{\sqrt{3}} \sqrt{\tr \! \left(\Rkt \Rlt \right)}} \nonumber \\ 
 & \quad - \PR{ \sigll > 5 \sqrt{\tr \! \left(\Rlt \Rlt \right)}} \nonumber \\ 
 &> 1 - e^{-\frac{1}{9}} - e^{-8} > \frac{1}{10}, \label{eq:sigmaFracLB}
\end{align}
where the second inequality follows from a union bound argument, and the third from 
\begin{equation}
\PR{\sigkl < \frac{1}{\sqrt{3}} \sqrt{\tr \! \left(\Rkt \Rlt \right)}} \leq e^{-\frac{1}{9}} \label{eq:SigklLowTailBnd}
\end{equation}
and 
\begin{equation}
\PR{\sigll > 5 \sqrt{\tr \! \left(\Rlt \Rlt \right)}} \leq e^{-8}, \label{eq:SigllUpTailBnd}
\end{equation}
both proven below. Inserting \eqref{eq:sigmaFracLB} into \eqref{eq:CondProbLB} yields the desired result.

\paragraph*{Proof of \eqref{eq:SigklLowTailBnd}\normalfont} We start by noting that $\sigkl^2 = \norm[2]{\smash{\transp{\Ck\!} \Cl \vyl_i}}^2 = \transp{\vyl_i\!} \transp{\Cl\!} \Rkt \Cl \vyl_i$ can be written as $\sigkl^2 \sim \sum_{m=1}^M \lambda_m z^2_m$, where $\lambda_m$, $m \in [M]$, denotes the non-negative eigenvalues of $\transp{\Cl\!} \Rkt \Cl$ and $z_m$, $m \in [M]$, are independent standard normal random variables. Setting $\boldsymbol \lambda = \transp{[\lambda_1 \, \dots \, \lambda_M]}$ and applying the lower tail bound \cite[Lem. 1]{laurent2000adaptive} for linear combinations of independent $\chi^2$ random variables yields, for $t > 0$, 
\begin{equation}
\PR{\sigkl^2 \leq \norm[1]{\boldsymbol \lambda} - 2 \norm[2]{\boldsymbol \lambda} \sqrt{t}} \leq e^{-t}. \label{eq:SigklsqChisqConc}
\end{equation}
The inequality \eqref{eq:SigklLowTailBnd} is obtained from \eqref{eq:SigklsqChisqConc} by noting that $\norm[2]{\boldsymbol \lambda} \leq \norm[1]{\boldsymbol \lambda}$ and $\norm[1]{\boldsymbol \lambda} = \tr \big( \transp{\Cl\!} \Rkt \Cl \big) = \tr \big( \Rkt \Cl \transp{\Cl\!} \big) = \tr \big( \Rkt \Rlt \big)$, and by setting $t = 1/9$ in \eqref{eq:SigklsqChisqConc}.

\paragraph*{Proof of \eqref{eq:SigllUpTailBnd}\normalfont} Noting that \vspace{0.06cm}$ \sigll = f(\vyl_i) = \norm[2]{\smash{\transp{\Cl\!} \Cl \vyl_i}} \allowbreak= \norm[2]{\smash{\Rlt \vyl_i}}$  is Lipschitz with Lipschitz constant $\norm[2 \to 2]{\smash{\Rlt}}$, we can invoke a well-known concentration inequality for Lipschitz functions of Gaussian random vectors with independent standard normal entries (see, e.g., \cite[Thm. 8.40]{foucart_mathematical_2013}) to get, for $t > 0$, 
\begin{align}
&\PR{\norm[2]{\Rlt \vyl_i} - \E{\norm[2]{\Rlt \vyl_i}} \geq t} \nonumber \\
&\hspace{3.7cm}\leq \exp \left(- \frac{t^2}{2 \norm[2 \to 2]{\smash{\Rlt}}^2} \right). \label{eq:SigllGaussConc}
\end{align}
The inequality \eqref{eq:SigllUpTailBnd} is now implied by $\E{\norm[2]{\smash{\Rlt \vyl_i}}} \leq \sqrt{\E{\norm[2]{\smash{\Rlt \vyl_i}}^2}} = \sqrt{\tr(\Rlt \Rlt)} = \norm[F]{\smash{\Rlt}}$ (where we used Jensen's inequality), $\norm[2 \to 2]{\smash{\Rlt}} \leq \norm[F]{\smash{\Rlt}}$, and \eqref{eq:SigllGaussConc} with $t = 4 \norm[F]{\smash{\Rlt}}$.

\bibliographystyle{jabbrv_ieeetr}
\bibliography{IEEEabrv,./processclustering.bib}

\end{document}